\documentclass[12pt]{article}
\usepackage{amsmath}
\usepackage{graphicx}
\usepackage{enumerate}
\usepackage{natbib}
\usepackage{url} 

\usepackage{xcolor}
\usepackage{graphicx}
\usepackage[colorlinks]{hyperref}
\usepackage{amsmath, amssymb, amsthm}
\usepackage{mathtools}
\usepackage{caption}
\usepackage{algorithm}
\usepackage{algorithmic}
\usepackage{comment}
\usepackage{enumitem}
\usepackage{xr}
\usepackage{comment}
\usepackage{enumitem}
\usepackage{subcaption}
\usepackage{booktabs}
\hypersetup{colorlinks,linkcolor={black},filecolor={black},citecolor={black},}

\newcommand{\E}{\mathbb{E}}
\newcommand{\prob}{\mathbb{P}}

\newcommand{\argmin}{\mathop{\arg\min}}
\newcommand{\argmax}{\mathop{\arg\max}}
\newcommand{\ld}{\mathcal{L}}
\newcommand{\ud}{\mathcal{U}}
\newcommand{\cd}{\mathcal{D}}

\newcommand{\sspace}{\mathcal{S}}  
\newcommand{\aspace}{\mathcal{A}}  

\newcommand{\obo}{\mathcal{B}^*}  
\newcommand{\oblo}{\mathcal{B}_{\ell}^*}  

\newcommand{\event}{\mathcal{E}}
\newcommand{\rub}{R_{\max}}
\newcommand{\vub}{V_{\max}}
\newcommand{\fclass}{\mathcal{Q}}
\newcommand{\indicator}{\mathbb{I}}
\newcommand{\rfun}{R}
\newcommand{\tfun}{\mathcal{P}}

\newcommand{\spl}{\textrm{{\tiny SPL}}}
\newcommand{\sug}{\textrm{{\tiny SUG}}}
\newcommand{\ini}{\textrm{{\tiny INI}}}
\newcommand{\aux}{\textrm{{\tiny AUX}}}
\newcommand{\taw}{\textrm{{\tiny TAW}}}
\newcommand{\ppl}{\textrm{{\tiny PPL}}}
\newcommand{\uds}{\textrm{{\tiny UDS}}}

\graphicspath{{figure/}}

\newtheorem{theorem}{Theorem}

\newtheorem{lemma}{Lemma}
\newtheorem{corollary}{Corollary}

\newtheorem{assumption}{Assumption}

\def\spacingset#1{\renewcommand{\baselinestretch}%
{#1}\small\normalsize} \spacingset{1}

\newcommand{\blind}{1}

\begin{document}

\if1\blind
{
\title{\Large{\textbf{Semi-pessimistic Reinforcement Learning}}} 
\author{
Jin Zhu$^1$, Xin Zhou$^2$, Jiaang Yao$^3$, Gholamali Aminian$^4$, \\
\bigskip
Omar Rivasplata$^5$, Simon Little$^3$, Lexin Li$^{2*}$, Chengchun Shi$^{1}\thanks{Co-corresponding authors}$ \\
\normalsize{$^1$London School of Economics and Political Science} \\
\normalsize{$^2$University of California at Berkeley, $^3$University of California at San Francesco} \\
\normalsize{$^4$The Alan Turing Institute, $^5$University of Manchester}
}
\date{}
\maketitle
} \fi

\if0\blind
{
\title{\Large{\textbf{Semi-pessimistic Reinforcement Learning}}}
\author{
\bigskip
\vspace{1.16in}
}
\date{}
\maketitle
} \fi

\spacingset{1.2} 
\begin{abstract}
Offline reinforcement learning (RL) aims to learn an optimal policy from pre-collected data. However, it faces challenges of distributional shift, where the learned policy may encounter unseen scenarios not covered in the offline data. Additionally, numerous  applications suffer from a scarcity of labeled reward data. Relying on labeled data alone often leads to a narrow state-action distribution, further amplifying the distributional shift, and resulting in suboptimal policy learning. To address these issues, we first recognize that the volume of unlabeled data is typically substantially larger than that of labeled data. We then propose a semi-pessimistic RL method to effectively leverage abundant unlabeled data. Our approach offers several advantages. It considerably simplifies the learning process, as it seeks a lower bound of the reward function, rather than that of the Q-function or state transition function. It is highly  flexible, and can be integrated with a range of model-free and model-based RL algorithms. It enjoys the guaranteed improvement when utilizing vast unlabeled data, but requires much less restrictive conditions. We compare our method with a number of alternative solutions, both analytically and numerically, and demonstrate its clear competitiveness. We further illustrate with an application to adaptive deep brain stimulation for Parkinson's disease. 
\end{abstract}
\bigskip


\newpage
\spacingset{1.8} 

\section{Introduction}

Reinforcement learning (RL) aims to train an intelligent agent to make optimal decisions in a given environment to maximize some long-term reward. It provides a powerful paradigm for policy optimization in sequential decision making \citep{sutton2018reinforcement}. Online RL learns an optimal policy by actively interacting with the environment and collecting new data through exploration and exploitation. By contrast, offline RL learns an optimal policy from a fixed, pre-collected dataset without further interaction with the environment. Offline RL offers distinct advantages in high-stakes applications, such as in healthcare, autonomous driving, and finance, when real-time data collection is expensive, risky, or impractical. However, it faces challenges of distributional shift, where the learned policy may encounter unseen scenarios not covered by the offline data. In addition, most existing offline RL algorithms are designed for settings where the rewards are fully observed and labeled, a presumption that often falls short in real-world applications \citep{Levine2020}.  

Our scientific motivation comes from deep brain stimulation (DBS), a highly effective treatment strategy for Parkinson's disease (PD) as well as other neurological disorders. DBS involves a neurosurgical procedure that implants electrodes in specific areas (basal ganglia) deep in the brain, which are connected to a battery-powered stimulator. When turned on, the stimulator sends electrical pulses to regulate the aberrant network signals. Since 1990s, DBS has been successfully used to mitigate and treat PD, shown to effectively reduce cardinal PD symptoms including tremor, rigidity and bradykinesia \citep{Okun2012}. Current clinical use of DBS has a trained physician to manually program the stimulation parameters, e.g., the stimulation amplitude. Once programmed and turned on, the device delivers constant electrical stimulation without regard to dynamic changes in neural activities or symptoms of patients, which can lead to unwanted side effects \citep{Neumann2023Adaptive}. More recently, a new DBS treatment strategy, namely, \emph{adaptive deep brain stimulation}, is emerging. It aims to adaptively adjust stimulation parameters in real-time, based on personalized neurophysiological biomarkers such as cortical or subcortical field potentials collected through the implanted electrode sensors, and individual patient's symptoms collected through wearable devices \citep{Little2013, Oehrn2024Chronic}. Adaptive DBS points to a more effective and promising direction of neuromodulation. 

RL provides a natural framework for adaptive DBS optimization. In the language of RL, the goal of adaptive DBS is to develop a stimulation policy (action) based on patient's signals (state) so as to optimize the treatment effect (reward). However, numerous challenges hinder the effective application of existing RL algorithms to the DBS setting. One such challenge is the limited availability of reward information. In our study, due to the device configuration, the patient's neural signals (state) are recorded every second, whereas the patient's Bradykinesia score (reward), which quantifies the severity of movement slowness resulting from the treatment, is measured every 120 seconds. As a result, there is only a limited amount of labeled reward data but a vast amount of unlabeled data. Relying on labeled data alone often leads to a narrow state-action distribution, and the distributional shift between the behavior policy and the optimal policy poses significant hurdles to effective policy learning. Another challenge is the relatively limited computing capacity of the stimulator device. These constraints call for an effective yet simple semi-supervised offline RL method that can fully leverage available data and improve policy optimization.

There is a rich literature on offline RL, including model-free approaches and model-based approaches; see, e.g., \citet{chen2019information,luckett2020estimating, mou2020,hao2021sparse, liao2022batch,li2023quasioptimal, wang2023projected,chen2024reinforcement,shi2024value, zhang2023estimation, zhou2024estimating, zhou2024reward}, among many others. Notably, there are a family of RL algorithms that employ the pessimistic principle to mitigate distributional shift in the offline setting \citep{yu2020mopo, jin2021pessimism, uehara2022pessimistic, jin2022policy, yin2022near, chen2023steel, li2024settling, wang2024pessimistic, wu2024neural}. There is also a rich literature on semi-supervised learning; see \citet{yang2023survey} for a review. Notably, pseudo labeling \citep{lee2013pseudo, kou2023how} stands out as a widely used technique thanks to its practical effectiveness, where the key idea is to generate labels for the unlabeled data based on a model learnt from the labeled data. There has also been some recent development integrating semi-supervised learning and offline RL, with two major lines of solutions. The first line imputes the reward with different imputation techniques, then applies a standard RL for policy learning \citep{konyushkova2020semi, sonabend2020semiope, sonabend2020semiopo, yu2022leverage, hu2023the}. The second line learns a state representation from the unlabeled data, then combines it with the labeled data for policy optimization \citep{yang2021representation, yang2022trail}. However, the existing solutions either rely on some restrictive conditions such as full coverage or linear Markov decision process, or require some complicated uncertainty quantification and uniform lower bound for the action-value function, i.e., the Q-function. These constraints have restricted the broader adoption of semi-supervised RL to applications such as adaptive DBS. 

In this article, we propose a new semi-supervised RL approach, which adopts what we term the \emph{semi-pessimistic} principle, and leverages both labeled and unlabeled data to tackle the challenges of distributional shift and missing reward. Our proposal enjoys several advantages. First, it stands out for its simplicity, as our key idea behind the semi-pessimism is to quantify the uncertainty of the reward function. In contrast, in full pessimism, it is to sequentially quantify the uncertainty of the series of estimated Q-functions or state transition functions, which is much more complex. This simplification makes our method more suitable for practical deployment in settings like adaptive DBS. Second, it is highly flexible. Once we obtain the imputed reward based on its quantified uncertainty, we apply a standard RL algorithm for policy learning, enabling the development of both model-free and model-based RL policy optimization algorithms. We illustrate with two such examples in our implementation. Third, we establish rigorous theoretical guarantees of the proposed RL algorithms, and we show that utilizing the vast unlabeled data can markedly improve both the accuracy of uncertainty quantification and the quality of policy learning. In addition, our theory requires much less restrictive conditions. It relies on a new core concept, the \emph{semi-coverage}, that we propose, which stands a middle ground in the spectrum of coverage. It is less stringent than the full coverage condition, and offers a tradeoff compared to the partial coverage. Moreover, it only requires a \emph{pointwise} uncertainty quantification, which is much weaker than the \emph{uniform} uncertainty quantification typically imposed in the literature. It does not require the Markov decision process to have to be linear either. Finally, we compare our proposed method with a range of existing semi-supervised offline RL solutions, both analytically and numerically. We demonstrate the competitiveness of our approach in synthetic and MuJoCo environments. We then revisit the motivating DBS study and show that the policy learnt by our method achieves a higher cumulative reward for patients, thus offering a promising solution for adaptive DBS for Parkinson's disease.

The rest of the article is organized as follows. Section \ref{sec:preliminary} lays out the problem setup and the pessimistic principle. Section \ref{sec:method} develops our semi-pessimistic RL approach. Section \ref{sec:regret} establishes the theoretical guarantees through a regret analysis. Section \ref{sec:numericals} presents the numerical studies, including the adaptive DBS application. Section \ref{sec:discussions} discusses some extensions. The Supplementary Appendix collects all technical proofs and additional results.

\section{Problem Setup and Pessimistic Principle}
\label{sec:preliminary}

\subsection{Problem setup}
\label{sec:setup}

We first lay out the problem setup. We model the sequential data over time as a time-homogeneous Markov decision process (MDP), and denote it as $(\sspace, \aspace, \rfun, \tfun, \gamma)$, where $\sspace$ and $\aspace$ are the state and action spaces, respectively, $\rfun(s,a)$ is the expected reward for a given state-action pair $(s, a)\in \sspace\times\aspace$, $\tfun(\cdot|s,a)$ specifies the conditional probability distribution of transitioning to a future state, and $\gamma\in (0,1)$ is the discount factor that balances the immediate and future rewards. For simplicity, we assume both the state and action spaces are discrete. When these spaces are continuous, we replace the probability mass functions with probability density functions. We consider the semi-supervised setting with two parts of the data: a labeled dataset, denoted as $\mathcal{L}$, containing a small set of (state, action, reward, next state) tuples that follow the aforementioned MDP, and an unlabeled dataset, denoted as $\mathcal{U}$, with a large set of (state, action, next state) tuples that follow the same MDP but without the reward component. Let $n_\ld$ and $n_\ud$ denote the size of $\ld$ and $\ud$, respectively.

For a given offline dataset $\mathcal{D}$, let $d_{\mathcal{D}}$ denote its state-action data distribution, i.e., $d_{\mathcal{D}}(s,a)=|\mathcal{D}|^{-1}\sum\limits_{(S,A)\in \mathcal{D}} \mathbb{I}(S=s,A=a)$, where $\mathbb{I}(\cdot)$ is an indicator function, $|\mathcal{D}|$ is the cardinality of $\mathcal{D}$, and the sum $\sum\limits_{(S,A)\in \mathcal{D}}$ is taken over all state-action pairs in $\mathcal{D}$. Let $\pi(\cdot|s):\mathcal{S} \rightarrow \mathcal{A}$ denote a policy that determines the conditional probability mass/density function of actions given a particular state $s$. Our goal is to identify the optimal policy $\pi^*$, defined as the maximizer of the $\gamma$-discounted expected cumulative reward, $J(\pi) = \sum_{t=0}^{\infty} \gamma^t \E^{\pi} \big[R(S_t, A_t)\big]$, where the expectation $\E^{\pi}$ is computed under the assumption that actions are selected according to $\pi$ such that $(S_t,A_t)$ is the state-action pair following the MDP and $\pi$ at time $t$. We define the Q-function under a given policy $\pi$ as $Q^\pi(s, a)= \sum_{t=0}^{\infty} \gamma^t \mathbb{E}^\pi [\rfun(S_t,A_t)| S_0=s, A_0=a]$, and use $Q^{*}$ to denote the optimal Q-function, i.e., the Q-function under $\pi^*$. 
Given a reference initial state distribution $\rho_0$ with full support over $\sspace$ and a policy $\pi$, let $d^{\pi}$ denote the $\gamma$-discounted visitation distribution, i.e., the probability distribution of visiting a state-action pair $(s,a)$ under $\pi$, assuming the initial state follows $\rho_0$, given by $(1-\gamma)\sum\limits_{t\ge 0}\gamma^t \prob^{\pi}(A_t=a,S_t=s|S_0\sim \rho_0)$.

\subsection{The pessimism principle and current challenges}
\label{sec:thepessimism}

Q-learning \citep{watkins1992q} is commonly employed in offline RL. However, its effectiveness depends on a critical full coverage condition, i.e., the state-action distribution presented in the offline dataset $\mathcal{D}$ must comprehensively cover the state-action distributions that would be induced by all  policies. This assumption is often difficult to satisfy in practice. When certain state-action pairs are less explored in $\mathcal{D}$, accurately estimating their Q-values becomes difficult, which in turn leads to overestimation of the Q-values, and ultimately results in a suboptimal policy.

The pessimistic principle is a key strategy that helps relax the full coverage condition. A common solution is to learn a pessimistic Q-function, denoted by $\widehat{Q}$, that lower bounds the true Q-value $Q^*$ with a high probability. This helps reduce the risk of overestimating the value of poorly sampled state-action pairs, and thus enhances the performance of the resulting policy $\widehat{\pi}$. However, obtaining a proper conservative $\widehat{Q}$ poses significant challenges. On the one hand, for all $(s, a)\in \sspace \times \aspace$, $\widehat{Q}$ needs to lower bound $Q^*$ with a high probability. On the other hand, it is important that $\widehat{Q}$ should not be excessively conservative, since the regret of the resulting estimated optimal policy $J(\pi^*)-J(\widehat{\pi})$ is proportional to the discrepancy between $Q^*$ and $\widehat{Q}$ \citep{jin2021pessimism}.

A popular approach to derive a tight lower bound is \emph{sequential} uncertainty quantification \citep[see e.g.,][]{jin2021pessimism,bai2022pessimistic,zhou2022optimizing}. This family of methods sequentially penalize the Q-function estimator by its uncertainty. Specifically, at the $k$th iteration, they learn a lower confidence bound $\widehat{Q}^{(k)}$ of the following Q-function (denoted by $Q^{(k)}$),
\begin{eqnarray*}
Q^{(k)}(s,a)=\mathcal{R}(s,a)+\gamma \E_{S'\sim \tfun(\cdot|s,a)}\max_a \widehat{Q}^{(k-1)}(S',a),
\end{eqnarray*}
where $\widehat{Q}^{(k-1)}$ is the lower confidence bound at the $(k-1)$th iteration. Nevertheless, such a sequential procedure introduces two additional complexities. First, the uncertainty quantification for $Q^{(k)}$ is particularly challenging, because $Q^{(k)}$ itself is parameter-dependent and relies on the previously estimated lower bounds. As a result, each iteration's outcome influences the subsequent estimation, creating a chain of dependency. Second, $\widehat{Q}^{(k)}$ must \emph{uniformly} lower bound $Q^{(k)}$ across all iterations $k$. Maintaining this uniformity is also challenging. For a significance level $\alpha\in (0,1)$, suppose $\widehat{Q}^{(k)}$ is a lower $\alpha$th confidence bound for $Q^{(k)}$ marginally for each $k$. Then, these bounds do not hold uniformly with probability $1-\alpha$. A potential solution is to apply the Bonferroni's correction, setting each $\widehat{Q}^{(k)}$ to a lower $(\alpha/K)$th bound, where $K$ is the total number of iterations. However, this leads to an overly conservative bound, particularly with a large $K$.

\section{Semi-pessimistic pseudo labeling}
\label{sec:method}

\subsection{Our proposal: key idea}
\label{sec:key-ideas}

Our key idea is to learn a lower bound for the reward function based on both labeled and unlabeled data, then devise a pessimistic reward as the pseudo label. We refer to our method as \emph{semi-pessimistic pseudo labeling} (SPL). Based on the pseudo labels, we develop both model-free and model-based semi-supervised offline RL algorithms accordingly. 

Our method is notable for its simplicity, and thus is particularly suitable for adaptive DBS type applications. It greatly differs from the existing pessimistic offline RL solutions, in that it avoids sequential quantification of the uncertainty of the Q-function, or the uncertainty of the state transition function. Our key observation is that, when the size of the unlabeled data increases, the impact of not accounting for the uncertainty in the Q-function or the state transition function is to diminish. We first illustrate with a simulation example, then formally introduce the notion of semi-coverage for a more rigorous analysis.

\begin{figure}[t!]
\centering
\includegraphics[width=0.9\linewidth]{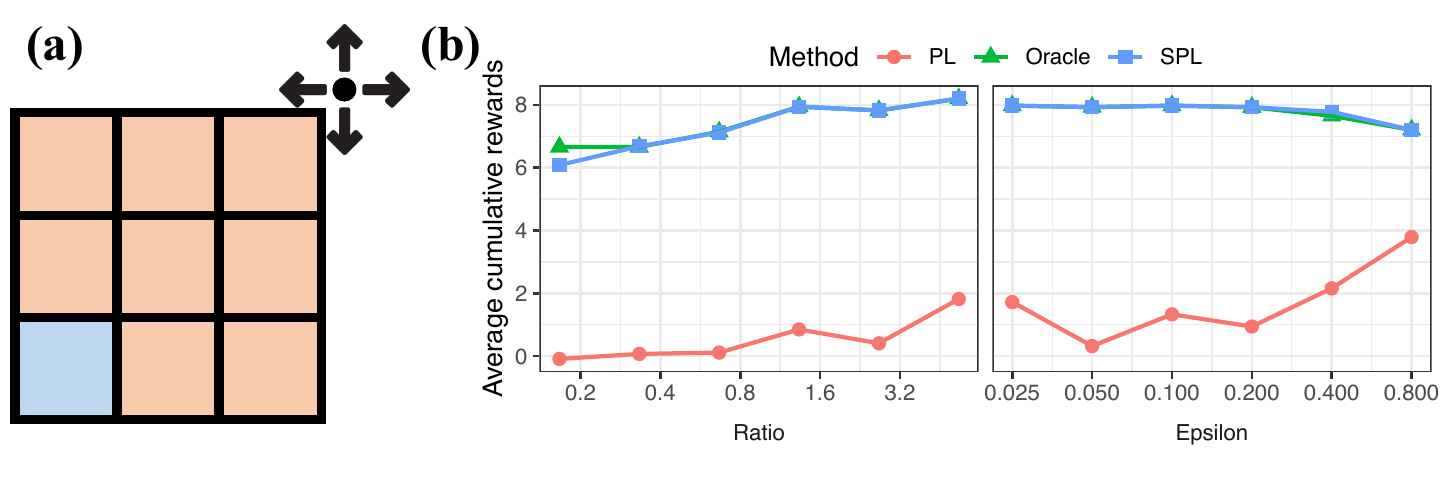}
\vspace{-10pt}
\caption{(a): Graphical illustration of the environment. (b): The average return with $n_\ld = 120$; left panel: the ratio $n_\ud / n_\ld$ (horizontal axis) varies; right panel: the value of $\epsilon$ (horizontal axis) varies with $n_\ud = 150$.}
\label{fig:toy-example}
\end{figure}

We consider a simple environment, where the state space is a $3 \times 3$ grid, and the action space comprises five discrete actions: $\{\textup{stay, right, up, left, down}\}$. The objective is to reach the target grid, located in the left-bottom corner, within two steps. The labeled data $\ld$ is generated by the $\epsilon$-greedy optimal policy, and the unlabeled data $\ud$ is generated through random exploration. More information regarding this environment is given in Appendix~\ref{sec:detail-toy-env}. Figure \ref{fig:toy-example}(a) visualizes the state and action spaces of this environment. Figure \ref{fig:toy-example}(b) shows the average returns over 20 replications of three methods, the pseudo labeling (PL) without utilizing the pessimism principle, an oracle pessimistic pseudo labeling method with a known transition function $\tfun$ (Oracle), and our proposed method (SPL). Figure \ref{fig:toy-example}(b), left panel, shows the average returns when the ratio of $n_\ud / n_{\ld}$ varies. In this case, we see that SPL clearly outperforms PL, demonstrating the advantage of adopting the semi-pessimistic principle. In addition, the oracle method is slightly better than SPL, but the advantage quickly diminishes as the size of $\ud$ increases. Figure \ref{fig:toy-example}(b), right panel, shows the average returns when the value of $\epsilon$ in the $\epsilon$-greedy algorithm for generating $\ld$ varies. Here a larger value of $\epsilon$ leads to a state-action distribution closer to that under random exploration, making the environment more likely to satisfy the full coverage condition. In this case, we see that SPL performs more closely to PL when $\epsilon$ is large, but much better when $\epsilon$ decreases, demonstrating again the importance of the semi-pessimistic principle when the full coverage condition is violated.

\subsection{Coverage conditions}
\label{sec:coveragecondition}

Nearly all existing offline RL algorithms rely on some forms of coverage conditions to consistently identify the optimal policy $\pi^*$. We first introduce the \emph{full coverage} condition that is typically required by non-pessimistic type offline RL algorithms, and the \emph{partial coverage} condition that is required by pessimistic type offline RL algorithms.

\noindent\textbf{Full coverage condition}. The data distribution covers the visitation distribution induced by any policy $\pi$, i.e., $\sup\limits_{\pi,s,a} d^{\pi}(s, a)/d_{\mathcal{D}}(s, a)<\infty$.
\bigskip

\noindent\textbf{Partial coverage condition}. The data distribution covers the visitation distribution induced by the optimal policy $\pi^*$, i.e., $\sup\limits_{s,a} d^{\pi^*}(s,a)/d_{\mathcal{D}}(s,a)<\infty$.

Next, we formally introduce the \emph{semi-coverage} condition, which offers a middle ground between the full and partial coverage conditions, and is essential for our method.

\begin{assumption}[\textbf{Semi-coverage condition}]\label{con:coverage}
Suppose (a) $B_{\ld}^* \equiv \sup\limits_{s,a} {d^{\pi^*}(s,a)} / {d_{\ld}(s,a)}<\infty$; and (b) $B_{\mathcal{D}} \equiv \sup\limits_{\pi,s,a} {d^{\pi}(s,a)} / {d_{\ld\cup \ud}(s,a)}<\infty$.
\end{assumption}

\noindent
We note that, this condition essentially requires a partial coverage on $\ld$, and a full coverage on the combination of $\ld$ and $\ud$. This is reasonable, because the unlabeled data is usually much larger in size than the labeled data, and thus the combined data contains a much wider range of state-action pairs than the labeled data alone. Additionally, as the size of the unlabeled data increases, the likelihood of satisfying (b) increases. Figure~\ref{fig:coverage} shows a simple illustration, which visualizes the state-action distribution of $\ld$, $\ud$ and $\ld \cup \ud$ on the discrete MDP with five states and three actions. Each state-action pair $(s, a)$ is colored green if $d_{\cd}(s, a) > 0$, indicating its presence in the data, and orange if absent. Regardless of the partial coverage of $d_{\ld}(s, a)$, the combination with $\ud$ results in a full coverage for $d_{\ld \cup \ud}(s, a)$. We also remark that we may relax this semi-coverage condition, and we discuss this extension with more details in Section \ref{sec:discussions}.

\begin{figure}[t!]
\centering
\includegraphics[width=0.7\linewidth,height=1.1in]{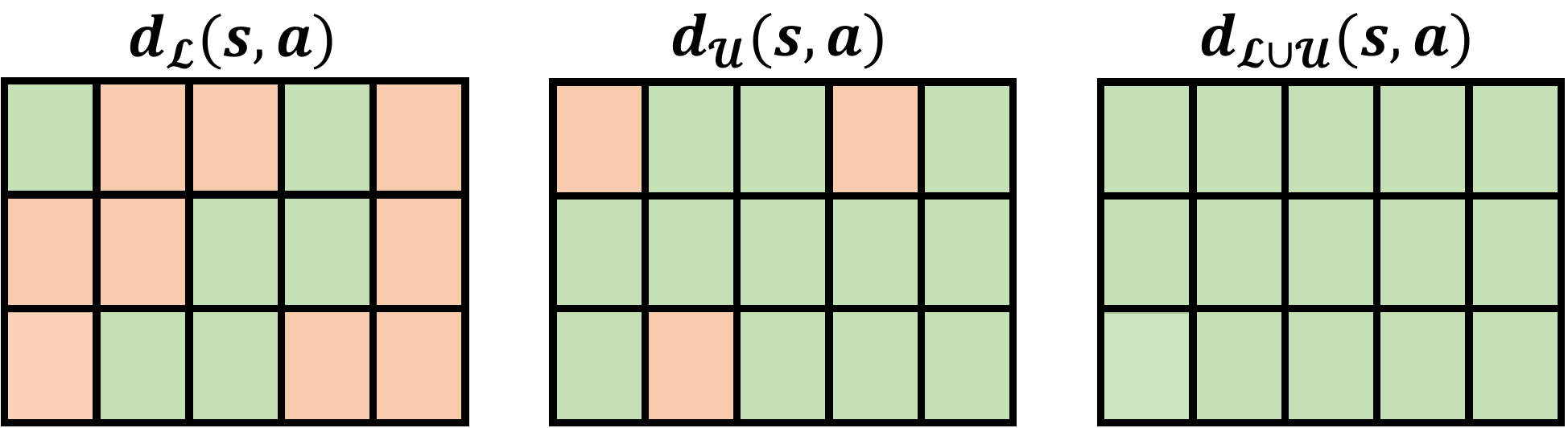}
\caption{Visualization of the state-action distribution of $\ld$, $\ud$ and $\ld \cup \ud$ on an MDP with five states (horizontal axis) and three actions (vertical axis). Each state-action pair $(s, a)$ is colored green if $d_{\cd}(s, a) > 0$, indicating its presence in the data, and orange if absent.}
\label{fig:coverage}
\end{figure}

Next, we show that the state-action distribution in the combined data sufficiently covers that in the labeled data, demonstrating that (b) is indeed weaker than the full coverage requirement on $\ld$ alone. Its proof is intuitive: every state-action pair presented in $\ld$ is also included in $\ld\cup \ud$, thus ensuring the coverage. 

\begin{lemma}\label{lemma:coverage}
For any $(s,a)\in\mathcal{S}\times\mathcal{A}$ with $d_{\ld}(s,a)>0$, we have $d_{\ld\cup \ud}(s,a)>0$. 
\end{lemma}

\subsection{Semi-supervised uncertainty quantification}
\label{sec:semi-supervised-uncertainty-quantification}

Next, we quantify the uncertainty $\Delta$ of the reward function $\rfun$, and construct a reward estimator $\widehat\rfun$ that lower bounds the reward $\rfun$ in the following sense. 

\begin{assumption}[Uncertainty quantification]\label{con:uncertainty}
For a significance level $\alpha$ and any state-action pair $(s, a)\in \mathcal{S}\times \mathcal{A}$, the event $\widehat\rfun(s,a) \le \rfun(s,a)$ holds with probability at least $1-\alpha$.
\end{assumption}

\noindent
We remark that, this condition only requires a \emph{pointwise} uncertainty quantification, i.e., the event holds marginally for each $(s, a)$ with probability $1-\alpha$. It can typically be achieved using a one-sided Wald interval. It is significantly weaker than the \emph{uniform} uncertainty quantification condition imposed in \citet{jin2021pessimism}, which instead requires the uncertainty quantification to hold uniformly across all $(s, a)$ with the same probability and is much more difficult to satisfy. Moreover, to ensure the regret of the estimated policy diminishes as the sample size increases, it is necessary for $\alpha$ to decrease to zero; see Corollary~\ref{coro1} in Section \ref{sec:regret}. 

We first present an initial uncertainty quantification that serves as a benchmark. Consider a simple model $g^\top(s, a)\theta$ for the reward $\rfun$ given the state $S$ and action $A$, where $g(s, a)$ is a nonlinear mapping from $\sspace \times \aspace$ to $\mathbb{R}^d$, and $\theta$ is the coefficient vector. In our implementation, we use the random Fourier features \citep{rahimi2007random} to construct $g$. We then obtain the ordinary least square (OLS) estimator $\widehat{\theta}_\ini$ for $\theta$, and the corresponding reward estimator  $\widehat\rfun_\ini(s,a) = g^\top(s, a)\widehat{\theta}_\ini$. The uncertainty of $\widehat\rfun_\ini(s,a)$ is quantified by 
\begin{equation}\label{eq:c-uncertainty}
\Delta_\ini(s, a) = \sqrt{n_\ld^{-1} g^\top(s, a) \ \widehat{\Sigma}_\ini \ g(s, a)},
\end{equation}
where $\widehat{\Sigma}_\ini$ is the sandwich variance estimator for $\widehat{\theta}_\ini$ \citep{kauermann2001note}. We remark that, for this uncertainty quantification method, only the labeled data $\ld$ is used. 

We next introduce the semi-supervised uncertainty quantification (SUQ) method we employ, following a similar strategy as in \citet{abhishek2018efficient,zhang2019semi,anastasios2023ppi}. The procedure consists of four main steps. In the first step, we fit an auxiliary and more flexible model, e.g., random forests or deep neural networks, for the reward given the state and action. In our implementation, we choose random forests. Denote the resulting estimator as $\widehat\rfun_\aux(s,a)$, and this fitting only uses the labeled data $\ld$. In the second step, we employ OLS to model the residual $\rfun - \widehat\rfun_\aux(s,a)$ given $g(s, a)$, again only using the labeled data $\ld$. Denote the resulting estimator as $\widehat\theta_\ld$, and its corresponding sandwich variance estimator as $\widehat\Sigma_\ld$. In the third step, we employ OLS to model $\widehat\rfun_\aux(s,a)$ given $g(s, a)$, but this time only using the unlabeled data $\ud$. This is doable because we can generate $\widehat\rfun_\aux(s,a)$ given the state-action pairs in the unlabeled data $\ud$ based on the auxiliary model learnt from the labeled data $\ld$ in the first step. Denote the resulting estimator as $\widehat\theta_\ud$, and its corresponding sandwich variance estimator as $\widehat\Sigma_\ud$. In the final step, we construct our SUQ estimator of $\theta$ and $\rfun$ as
\begin{eqnarray}\label{eq:ppi-est}
\widehat\theta_\sug = \widehat\theta_\ld + \widehat\theta_\ud, \quad\quad \widehat\rfun_\sug(s,a) = g^\top(s, a)\widehat{\theta}_\sug.
\end{eqnarray}
The uncertainty of $\widehat\rfun_\sug(s,a)$ is quantified by
\begin{equation}\label{eq:ppi-uncertainty}
\Delta_\sug(s, a) = \sqrt{n_\ld^{-1} g^\top(s,a) \ \widehat{\Sigma}_{\ld} \ g(s,a) + n_\ud^{-1} g^\top(s,a) \ \widehat{\Sigma}_{\ud} \ g(s,a)}.
\end{equation}

Given $\widehat\rfun_\sug$ and its uncertainty quantification $\Delta_\sug(s, a)$, we impute the missing reward in $\ud$ with the following pessimistic reward function,
\vspace{-0.01in}
\begin{equation} \label{eq:pessimistic-reward}
\widehat\rfun_\spl(s, a) = \widehat\rfun_\sug(s, a) - z_{1-\alpha/2}\Delta_\sug(s, a),
\end{equation}
where $z_{1-\alpha/2}$ is the $(1-\alpha/2)$th quantile of standard Gaussian distribution and $\alpha$ is the significance level. Following \citet{anastasios2023ppi}, we can show that $\widehat\rfun_\spl(s, a)$ satisfies Assumption \ref{con:uncertainty}, provided that OLS offers a reasonable approximation of the true reward. We summarize the above procedure in Algorithm \ref{alg:ppi-uncertainty}.

\begin{algorithm}[t!]
\caption{Semi-supervised uncertainty quantification.}
\label{alg:ppi-uncertainty}
\begin{algorithmic}[1]
\STATE Train an auxiliary model for the reward given the state and action with a flexible machine learning method, e.g., random forests or deep neural networks, using only the labeled data $\ld$; denote the resulting estimator as $\widehat\rfun_\aux(s,a)$.
\STATE Use OLS to model the residual $\rfun - \widehat\rfun_\aux(s,a)$ given $g(s, a)$, using only the labeled data $\ld$. Denote the resulting estimator as $\widehat\theta_\ld$, and its sandwich variance estimator as $\widehat\Sigma_\ld$.
\STATE Use OLS to model $\widehat\rfun_\aux(s,a)$ given $g(s, a)$, using only the unlabeled data $\ud$. Denote the resulting estimator as $\widehat\theta_\ud$, and its sandwich variance estimator as $\widehat\Sigma_\ud$.
\STATE Compute the SUG estimator $\widehat\rfun_\sug(s,a)$ using \eqref{eq:ppi-est}, the uncertainty quantification $\Delta_{\textup{SUQ}}(s, a)$ using \eqref{eq:ppi-uncertainty}, and the SPL reward estimator $\widehat\rfun_\spl(s, a)$ using \eqref{eq:pessimistic-reward}. 
\end{algorithmic}
\end{algorithm}

Comparing SUG with the benchmark reveals additional insight on the advantages of utilizing both labeled and unlabeled data. We first compare the coefficient estimators $\widehat\theta_\sug$ and $\widehat\theta_\ini$. It is straightforward to see that $\widehat\theta_\sug$ can be rewritten as 
\begin{eqnarray*}\label{eq:theta-rewrite}
\widehat\theta_\sug = \widehat\theta_\ini - \widehat\vartheta_\ld + \widehat\theta_\ud, 
\end{eqnarray*}
where $\widehat\vartheta_\ld$ is the OLS estimator of modeling $\widehat\rfun_\aux(s,a)$ given $g(s, a)$ using only the labeled data $\ld$. Since $\widehat\vartheta_\ld$ and $\widehat\theta_\ud$ share the same mean, the inclusion of them with opposite signs ensures the unbiasedness asymptotically. Meanwhile, $\widehat\vartheta_\ld$ helps reduce the variance of $\widehat\theta_\ini$, and the variance of $\widehat\theta_\ud$ becomes negligible due to the large amount of unlabeled data.

We next compare the uncertainty quantifications $\Delta_\sug(s, a)$ in \eqref{eq:ppi-uncertainty} and $\Delta_\ini(s, a)$ in \eqref{eq:c-uncertainty}. The former utilizes both the labeled data $\ld$ and the unlabeled data $\ud$, while the latter only utilizes $\ld$. Consequently, our quantification $\Delta_\sug$ is expected to be tighter compared with $\Delta_\ini$. This is because, when $n_\ud \gg n_\ld$, \eqref{eq:ppi-uncertainty} is asymptotically equivalent to $\sqrt{n_\ld^{-1} g^\top(s,a) \ \widehat{\Sigma}_{\ld} \ g(s,a)}$. Besides, $\widehat{\Sigma}_{\ld} \prec \widehat{\Sigma}$ in the semidefinite order when the prediction $\widehat\rfun_\aux(S, A)$ accounts for a good portion of the variation in the reward. In other words, when the size $n_\ud$ is large and the prediction $\widehat\rfun_\aux$ is accurate, $\Delta_\sug$ would be much smaller than $\Delta_\ini$. This observation is formally characterized by the next lemma. 

\begin{assumption}[Auxiliary model]\label{assump:auxmodel}
Suppose $\widehat{\rfun}_{\aux}(S,A)$ can be represented by $\widehat{\rfun}_{\aux}(S,A)$ $= \beta_0 R(S,A)+\beta_1^\top g(S,A)+e$, where the error term $e$ satisfies that $\mathbb{E}_{(S,A)\sim \ld} [e|R(S,A), g(S,A)]=0$. Let $\theta^*_{\textup{\ini}}$ denote the population limit of $\widehat{\theta}_{\textup{\ini}}$. Suppose $0< \beta_0< 2$, and
\begin{align}\label{eqn:someinequality}
\begin{split}
\mathbb{E}_{(S,A)\sim \ld} & \left[ g(S,A)  g^\top(S,A)e^2 \right] \prec  \\
& \left( 2\beta_0-\beta_0^2 \right) \mathbb{E}_{(S,A)\sim \ld} \left\{ g(S,A)g^\top(S,A)\left[ R(S,A)-g(S,A)\theta^*_{\textup{\ini}} \right]^2 \right\}.
\end{split}
\end{align} 
\end{assumption}

\begin{lemma}\label{lemma:ppi-adv-formal}
Let $\Delta_{\textup{\sug}}^*(s,a)$ and $\Delta_{\textup{\ini}}^*(s,a)$ denote the population limit of $\Delta_{\textup{\sug}}(s,a)$ and $\Delta_{\textup{\ini}}(s,a)$, respectively, where the sandwich estimators $\widehat{\Sigma}_{\ld}$ and $\widehat{\Sigma}_{\ud}$ are replaced with the true covariances. Suppose Assumption \ref{assump:auxmodel} holds. Then, as $n_\ld/n_\ud\to 0$, we have $\Delta^*_\textup{SUG}(s, a) \leq \Delta_{\textup{\ini}}^*(s, a)$, for any $(s, a) \in \sspace \times \aspace$, with the equality holding only if $g(s,a) = 0$.
\end{lemma}

\noindent 
Assumption \ref{assump:auxmodel} is mild. It holds as long as $\widehat{\rfun}_{\aux}(S,A)$ can be written as a combination of $R(S,A)$ and $g(S,A)$ plus an approximation error term $e$, while \eqref{eqn:someinequality} implies that $e$ is small relative to the error incurred by using $g(S,A)$ alone to approximate $R(S,A)$. Then, under this assumption, Lemma \ref{lemma:ppi-adv-formal} formally shows that SUG yields a tighter uncertainty quantification. Later, as we show both theoretically and numerically, a tighter uncertainty quantification leads to a smaller regret and a better RL policy. 

Finally, we remark that our method uses OLS to estimate the coefficients. This is primarily for simplicity, as it offers great computational advantages and is well suited for adaptive DBS type applications. It is possible to use alternative estimation methods, by replacing the squared error loss with a more flexible loss, though our numerical experiments show similar empirical performance. We thus adhere to OLS in our implementation.

\subsection{Two semi-supervised RL algorithms}
\label{sec:semiRLalgo}

Given the imputed pessimistic reward function $\widehat\rfun_\spl$, and the tuples $\Big\{(S,A,\widehat{\rfun}_\spl(S, A),S'):(S,A,S') \in \ld\cup \ud \Big\}$, we can couple our method with any existing RL algorithms. We next present two such algorithms, a model-free one in Algorithm \ref{alg:mf-sorl} that is based on FQI \citep{ernst2005tree,riedmiller2005neural}, and a model-based one in Algorithm \ref{alg:mb-sorl} that is based on MOPO \citep{yu2020mopo}. The main difference is that the former does not estimate the transition function $\tfun$, while the latter does. Later we analyze these two algorithms and establish their theoretical guarantees. Additionally, we compare them with existing solutions and demonstrate their empirical competitiveness. 

\begin{algorithm}[t!]
\caption{The model-free semi-pessimistic pseudo labeling algorithm}
\label{alg:mf-sorl}
\begin{algorithmic}[1]
\STATE Compute the pessimistic reward function $\widehat{\rfun}_{\spl}$ via Algorithm~\ref{alg:ppi-uncertainty}.
\STATE Initialize $\widehat{Q}^{(0)} \leftarrow 0$ and set $K$ as the maximum number of iterations.
\STATE For $k=1,\ldots,K$, employ supervised learning to compute
\begin{equation*}
\widehat{Q}^{(k)} = \argmin_{Q\in \mathcal{Q}}\sum_{(S,A,S')\in \ld\cup \ud} \left\{ \widehat{\rfun}_{\spl}(S,A)+\gamma \max_a \widehat{Q}^{(k-1)}(S',a) - Q(S,A) \right\}^2.
\vspace{-0.2in}
\end{equation*}
\STATE Output $\widehat{\pi}$ as the greedy policy of $\widehat{Q}^{(K)}$, i.e., $\widehat{\pi}(a|s) = 1$, if $a = \argmax\limits_{a'}\widehat{Q}^{(K)}(s, a')$, and $0$, otherwise.
\end{algorithmic}
\end{algorithm}

\begin{algorithm}[t!]
\caption{The model-based semi-pessimistic pseudo labeling algorithm}
\label{alg:mb-sorl}
\begin{algorithmic}[1]
\STATE Get the pessimistic reward function $\widehat{\rfun}_{\spl}$ via Algorithm~\ref{alg:ppi-uncertainty}.
\STATE Train the estimated probability transition $\widehat{\tfun}(\cdot | s, a) = \mathcal{N}(\mu(s, a), \Sigma(s, a))$ based on $\ld \cup \ud$ using neural networks, and construct the MDP, $( \sspace, \aspace, \widehat{\rfun}_{\spl}, \widehat\tfun, \gamma)$.
\STATE Initialize policy $\widehat{\pi}$ and empty replay buffer $\mathcal{D}' \leftarrow \emptyset$.
\REPEAT
\STATE Sample a state $S$ from $\ld \cup \ud$.
\STATE \textbf{While not converged do}
\STATE \quad Sample an action by $A \leftarrow \pi(S)$.
\STATE \quad Roll out the reward $R \leftarrow \widehat{\rfun}_{\spl}(S, A)$, and the next state: $S' \leftarrow \widehat{\tfun}(\cdot|S, A)$.
\STATE \quad Update replay buffer: $\mathcal{D}' \leftarrow \mathcal{D}' \cup \{(S, A, R, S')\}$ and update state: $S \leftarrow S'$.
\STATE \textbf{end}
\STATE Use soft actor critic to update $\widehat{\pi}$ with samples in $\ld \cup \ud \cup \mathcal{D}'$.
\UNTIL{$\widehat{\pi}$ convergence}
\STATE Output $\widehat{\pi}$.
\end{algorithmic}
\end{algorithm}

\section{Regret Analysis}
\label{sec:regret}

\subsection{Regularity conditions}
\label{sec:regularity-conditions}

Our proposed method bypasses the complex task of quantifying the uncertainty of the estimated Q-function or state transition function, whereas our pointwise uncertainty quantification is adequate for achieving the asymptotically zero regret, as we show next. We begin with a set of regularity conditions.

\begin{assumption}[Boundedness]\label{con:value-func}
Suppose the absolute value of the immediate reward and $\sup\limits_{s,a}|\widehat{\rfun}_{\spl}(s,a)|$ are upper bounded by some constant $R_{\max}$. Additionally, for the model-free algorithm, suppose the Q-function class $\mathcal{Q}$ is uniformly bounded by $\vub = {\rub}/{(1-\gamma)}$.
\end{assumption}

\begin{assumption}[Exponential $\beta$-mixing]\label{con:indep-tuple}
The tuples in $\ld \cup \ud$ are exponentially $\beta$-mixing \citep{bradley1986basic} with the mixing coefficients $\beta(q) \leq \kappa \rho^{q}$, for $\kappa>0$, $\rho \in (0, 1)$, and $q \geq 0$. 
\end{assumption}

\begin{assumption}[Finite hypothesis class]\label{con:func-class}
Suppose $\widehat{\rfun}_{\spl} \in \mathcal{F}$ almost surely for some finite hypothesis class $\mathcal{F}$. Additionally, for the model-based algorithm, suppose  $\fclass$ is a finite hypothesis class.
\end{assumption}

\begin{assumption}[Completeness]\label{con:completeness}
Suppose the function $f(s,a)+\gamma \sum\limits_{s'}\tfun(s'|s,a)\max\limits_{a'}Q(s',a')$ belongs to $\mathcal{Q}$, for any $f \in \mathcal{F}$ and $Q\in \mathcal{Q}$.
\end{assumption}

\noindent
We remark that conditions similar to Assumptions \ref{con:value-func} to \ref{con:completeness} have been widely imposed in the RL literature to facilitate the theoretical analysis \citep[see, e.g.,][]{chen2019information,liu2020provably,shi2022statistical,uehara2022pessimistic,liu2023online,ramprasad2023online}. We also remark that these conditions can be relaxed. In particular, Assumption \ref{con:value-func} can be replaced with some tail conditions, such as the sub-exponential or the sub-Gaussian conditions, thereby allowing for unbounded immediate rewards. Assumption \ref{con:func-class} can be relaxed by restricting $\mathcal{F}$ and $\mathcal{Q}$ to the Vapnik–Chervonenkis (VC) classes similarly as in \citet{uehara2021finite}, allowing them to contain infinitely many elements. Assumption \ref{con:completeness} can be relaxed by allowing the approximation error $\varepsilon_{\mathcal{Q}}=\inf_{Q'}\sup_{f,Q} \|\obo(f,Q)-Q'\|_{d_{\ld\cup \ud}}>0$, where $\obo(f, Q) = f(s,a)+\gamma \sum_{s'}\tfun(s'|s,a)\max\limits_{a'}Q(s',a')$, and $\|Q_1-Q_2\|_{d_{\mathcal{L}\cup \mathcal{U}}} = \left[ \E_{(S,A)\sim d_{\mathcal{L}\cup \mathcal{U}}} |Q_1(S,A)-Q_2(S,A)|^2 \right]^{1/2}$ for any $Q_1,Q_2 : \sspace\times \aspace\to \mathbb{R}$. In our proofs, we adhere to the current forms of these conditions for clarity and simplicity.

\subsection{Main theorems}

We now establish the regret of the estimated optimal policy $\widehat{\pi}$ from Algorithms \ref{alg:mf-sorl} and \ref{alg:mb-sorl}.

\begin{theorem}[Regret of the model-free algorithm]\label{thm:mf-fqi}
Suppose Assumptions~\ref{con:coverage}, \ref{con:uncertainty}, and \ref{con:value-func} to \ref{con:completeness} hold. Then the regret of the optimal policy $\widehat{\pi}$ from Algorithm~\ref{alg:mf-sorl}, $\E[J(\pi^*) - J(\widehat{\pi})]$, is upper bounded by
\begin{equation} \label{eqn:regret-model-free}
\frac{2\gamma^K V_{\max}}{1-\gamma}+\frac{2\alpha R_{\max}}{(1-\gamma)^2}+\frac{c_1\sqrt{B_{\ld}^*}}{(1-\gamma)^2}\|\rfun-\widehat{\rfun}_{\spl}\|_{d_{\ld}}
+\frac{c_1V_{\max}\sqrt{\ln(n) B_{\mathcal{D}} [\ln (|\mathcal{F}||\mathcal{Q}|)+\ln(n)] }}{(1-\gamma)^2\sqrt{n}},
\end{equation}
for some constant $c_1>0$, where $n=n_{\ld}+n_{\ud}$, $K$ is the number of FQI iterations, and $\|f \|_{d_{\ld}} = \left\{ \E_{(S, A)\sim d_{\ld}} f^2(S, A) \right\}^{\frac{1}{2}}$. 
\end{theorem}

\noindent
The regret bound in \eqref{eqn:regret-model-free} consists of four terms, each reflecting a specific aspect of the learning process. We discuss them one by one. 

The first term represents the initialization bias induced by the difference between the initial Q estimator $\widehat{Q}^{(0)}$ and the optimal $Q^*$. This bias decays to zero exponentially fast with the number of iterations $K$. It becomes negligible by choosing a sufficiently large $K$.

The second term is proportional to the type-I error $\alpha$, and upper bounds the probability that the lower bound $\widehat{\rfun}_\spl(s,a)$ exceeds the oracle reward $\rfun(s,a)$. This term can be made substantially small by employing concentration inequalities, meanwhile without causing a substantial increase in the reward estimation error. 

The third term essentially characterizes the reward estimation error, and relies on the discrepancy between $\widehat{\rfun}_{\spl}$ and $\rfun$. Notably, this term is proportional to $B_{\ld}^*$, also referred to as the single-policy concentration coefficient, requiring $d_{\ld}$ to cover only $\pi^*$-induced state-action distribution, thus relaxing the full coverage assumption on $\ld$.

The last term measures the supervised learning estimation error inherent in each iteration's application of the supervised learning algorithm to compute $\widehat{Q}^{(k)}$. It relies on the full coverage on $\ld \cup \ud$, which is likely to hold with a large $\ud$. Moreover, the presence of $\ln (n)$ in the last term accounts for the temporal dependence of tuples. It increases the supervised learning error by a factor logarithmic in the sample size.

\begin{theorem}[Regret of the model-based algorithm]\label{thm:mb}
Suppose Assumptions \ref{con:coverage}, \ref{con:uncertainty}, and \ref{con:value-func} hold. Then the regret of the optimal policy $\widehat{\pi}$ from Algorithm~\ref{alg:mb-sorl}, $\E[J(\pi^*) - J(\widehat{\pi})]$, is upper bounded by
\begin{equation}\label{eqn:regret-model-based}
\frac{2\alpha R_{\max}}{(1-\gamma)^2} + c_2\left[\frac{\sqrt{B_{\ld}^*}\|\rfun-\widehat{\rfun}_{\spl}\|_{d_{\ld}}}{(1-\gamma)^2}+\frac{\gamma V_{\max} \sqrt{B_{\mathcal{D}}}}{(1-\gamma)^2}\|\tfun-\widehat{\tfun}\|_{d_{\ld\cup \ud}}\right],
\end{equation}
for some constant $c_2>0$, where $\|\tfun-\widehat{\tfun}\|_{d_{\ld\cup \ud}} \coloneqq \sqrt{\E\E_{(S,A)\sim d_{\ld\cup \ud}} \textup{TV}_{\widehat{\tfun}}^2(S,A)}$, where $\textup{TV}_{\widehat{\tfun}}(s,a)$ denotes the total variation distance between the two conditional distributions $\widehat{\tfun}(\cdot|s,a)$ and $\tfun(\cdot|s,a)$ and the outer expectation is taken over the estimated transition $\widehat{\tfun}$.
\end{theorem}

\noindent
The regret bound in \eqref{eqn:regret-model-based} consists of three terms, corresponding to the type-I error term, the reward estimation error term, and the transition estimation error term, respectively. The first two terms correspond to the second and third terms in \eqref{eqn:regret-model-free}. We further elaborate on the last term, which depends on the model involved. Consider a discrete MDP, where the state transition function can be estimated by its empirical distribution. Then we can show that $\|\tfun-\widehat{\tfun}\|_{d_{\ld\cup \ud}} \leq \sqrt{2n^{-1}|\sspace|^3|\aspace|\log(\delta^{-1}|\sspace||\aspace|)}$, with probability at least $1 - \delta$, where the convergence rate of the upper bound is $O\big(n^{-1/2}\big)$. As a result, with a large amount of unlabeled data, this term can be much smaller than the reward estimation error, which typically scales as $O\big(n_\ld^{-1/2}\big)$.

Finally, to further facilitate the understanding of our theory, we consider a hypothetical scenario with an infinite amount of unlabeled data. 

\begin{corollary}\label{coro1}
Suppose the conditions in Theorems \ref{thm:mf-fqi} and \ref{thm:mb} hold. Then, with an infinite amount of unlabeled data, a sufficiently large $K$, and a sufficiently small $\alpha$, the regrets of the policies from Algorithms \ref{alg:mf-sorl} and \ref{alg:mb-sorl} satisfy that $\E [J(\pi^*)-J(\widehat{\pi})]=O\big(\sqrt{B_{\ld}^*}\|\rfun-\widehat{\rfun}_{\spl}\|_{d_{\ld}}\big)$.
\end{corollary}

\noindent
This corollary \ref{coro1} shows that, with an infinite amount of unlabeled data, the regrets of policies derived from both model-free and model-based algorithms are dominated by the reward estimation error, which is typically proportional to the uncertainty quantification $\Delta$, and thus a smaller $\Delta$ leads to a smaller regret. Consequently, it validates our method that leverages both labeled and unlabeled data, since it achieves a tighter uncertainty quantification compared to the method that only utilizes the labeled data.

\subsection{Analytic comparison to existing solutions}
\label{sec:a-compare}

We next compare our method to some state-of-the-art solutions analytically. Later in Sections \ref{sec:numericals}, we compare them numerically.

The first method to compare is pseudo labeling \citep[PL]{konyushkova2020semi}. PL is a widely used technique in semi-supervised learning. A key difference is that it uses $\ld$ to obtain a consistent reward estimator $\widehat{\rfun}$, instead of a pessimistic reward estimator $\widehat{\rfun}_{\spl}$. Consequently, it requires the full coverage condition on $\ld$, which is more restrictive than our semi-coverage condition. To further see this, following a similar argument as the proofs of Theorems~\ref{thm:mf-fqi} and \ref{thm:mb}, we can show that a non-pessimistic PL-type algorithm would incur a regret of the order $\sqrt{B_{\ld}}\|\rfun-\widehat{\rfun}\|_{d_{\ld}}$, where $B_{\ld} = \sup\limits_{s,a,\pi} d^{\pi}(s,a)/d_{\ld}(s,a)$ is the uniform concentration coefficient. A full coverage on $\ld$ is required to ensure that $B_{\ld}$ is bounded. In contrast, we only require the semi-coverage to ensure that $B^*_{\ld}$ is bounded.

The second method is unlabeled data sharing \citep[UDS]{yu2022leverage}. UDS is a semi-supervised offline RL algorithm. It imputes the missing rewards in $\ud$ using the minimal reward estimated from $\ld$, then applies an existing offline RL algorithm to the combined data for policy learning. A key difference is that it uses the \emph{minimal} reward, which may be overly conservative. As Corollary \ref{coro1} suggests, the bound of regret is proportional to the reward estimation error. Since UDS uses the minimal reward for imputation, the error $\|\rfun-\widehat{\rfun}_{\uds}\|_{d_{\ld}}$ can be lower bounded by a constant. Specifically, its reward estimation function can be written as $\widehat{\rfun}_{\uds}(s, a) = 
\inf\limits_{s, a} \rfun(s, a)$. Then, $\|\rfun - \widehat{\rfun}_{\uds}\|_{d_{\ld}} \geq \E_{(S, A) \sim d_{\ld}}\left[\rfun(S, A) - \inf\limits_{s, a}\rfun(s, a)\right]$. In contrast, we leverage both $\ld$ and $\ud$ to carefully construct a tight lower bound for the reward, and the error $\|\rfun-\widehat{\rfun}_{\spl}\|_{d_{\ld}}$ can be upper bounded by $O_p(n_{\ld}^{-1/2})$ up to some logarithmic factor, and this upper bound goes to zero as $n_{\ld}$ increases. This proves the superiority of our method.

The third method is provable data sharing \citep[PDS]{hu2023the}. Similar to our method, PDS also begins by constructing a pessimistic reward estimator. However, a key difference is that it then learns a policy using a pessimism-based offline RL algorithm, which still involves the inherent complexity in quantifying the uncertainty of the Q-function. In contrast, our method avoids such a complex task. Besides, we leverage $\ud$ to enhance the uncertainty qualification, leading to a smaller standard error of $\widehat{\rfun}$, and consequently, a smaller reward estimation error. Theoretically, PDS requires a linear MDP condition, but we do not, while our algorithm achieves the regret of the order $O(\|\rfun - \widehat{\rfun}_{\spl}\|_{d_{\ld}})$ that is comparable to that of PDS.

The last family of methods are pessimistic RL \citep{yu2020mopo,jin2021pessimism}. Existing pessimistic RL algorithms generally utilize the labeled data \emph{only}, and involve the complex tasks of uncertainty qualification of the Q-function or the state transition function. Besides, the regret typically relies on both the reward estimation error and the supervised learning estimation error on $\tfun$ or the Q-function. In contrast, as Corollary~\ref{coro1} suggests, both regrets in Theorems \ref{thm:mf-fqi} and \ref{thm:mb} are dominated by the reward estimation error only, highlighting the benefit of incorporating $\ud$ in policy learning.

\section{Numerical Studies}
\label{sec:numericals}

\subsection{A synthetic environment}
\label{sec:synthetic-env}

To assess the empirical performance of our method, we first consider a synthetic environment with continuous states and discrete actions.  For this environment, we illustrate with the model-free Algorithm \ref{alg:mf-sorl}. We consider two settings: a full coverage setting, where the labeled data is generated under a uniform random policy, and a partial coverage setting, where 80\% of the tuples in the labeled dataset with sub-optimal actions are removed. More information regarding this environment and the implementation is given in Appendix \ref{sec:detail-synthetic-env}. 

We  compare our proposed method SPL with a number of alternative solutions, including PL, UDS, and PDS analyzed in Section \ref{sec:a-compare}. In addition, we compare with NoShare, which applies FQI solely to the labeled data, and PNoShare, a variant that also applies FQI to the labeled data but with a pessimistic reward. Note that PNoShare can be viewed as a member of the pessimistic RL family analyzed in Section \ref{sec:a-compare}. For a fair comparison, all methods employ FQI as the base RL algorithm, and the evaluation criterion is the average regret over 100 replications. 

\begin{figure}[t!]
\centering
\includegraphics[width=1.0\columnwidth]{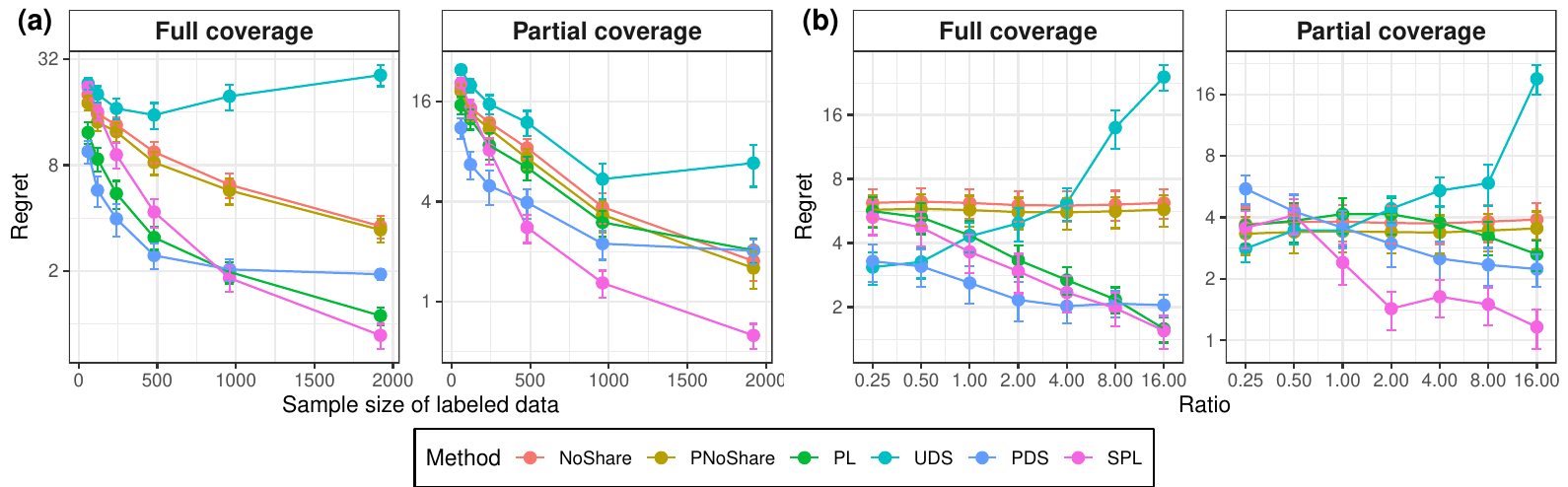}
\caption{Average regret of the policy learned by numerous model-free RL algorithms in the synthetic environment. (a): the size of the labeled data $n_\ld$ (horizontal axis) varies with $n_\ud / n_\ld$ fixed at 10; (b) the ratio $n_\ud / n_\ld$ (horizontal axis) varies with $n_\ld$ fixed at 32.}
\label{fig:FQI}
\end{figure}

Figure~\ref{fig:FQI}(a) reports the results for all methods when $n_\ld$ increases with $n_\ud / n_\ld$ fixed at 10, whereas Figure~\ref{fig:FQI}(b) reports the results when the ratio $n_\ud / n_\ld$ increases with $n_\ld$ fixed at 32. We see from the plot that PDS and our proposed SPL generally achieve the lowest regret. 
PDS only relies on labeled data for uncertainty quantification, and it performs worse than our SPL when the amount of either labeled or unlabeled data increases. NoShare and PL lag behind, as neither incorporates the pessimistic principle, while PNoShare also falls behind, as it only uses labeled data. Lastly, UDS exhibits increasing regret as the amount of unlabeled data grows, due to its overly conservative reward estimation.

\subsection{MuJoCo environments}
\label{sec:mujoco-env}

We next consider the MuJoCo environments with both continuous states and actions, including halfcheetah, walker2d, and hopper. For these environments, we illustrate with the model-based Algorithms \ref{alg:mb-sorl}. The label datasets consist of the demonstration of fully trained policies. As for the unlabeled dataset $\ud$, we consider two settings: full replay where $\ud$ consists of recording of all samples in the replay buffer when the policy reaches an expert level, and medium replay where $\ud$ consists of recording of all samples in the replay buffer observed during training until the policy reaches the medium level of performance.  More information regarding these environments and the implementation is given in Appendix~\ref{sec:detail-mujoco-env}. We again compare our method with the same set of alternative solutions in Section \ref{sec:synthetic-env}. For a fair comparison, all methods, except for PDS, employ MOPO as the base RL algorithm, and the evaluation criterion is average cumulative reward over 100 replications. Given that the practical implementation of PDS builds on the model-free FQI \citep{hu2023the}, we choose not to compare with it in this example. 

\begin{figure}[t!]
\centering
\includegraphics[width=0.8\linewidth, height=3in]{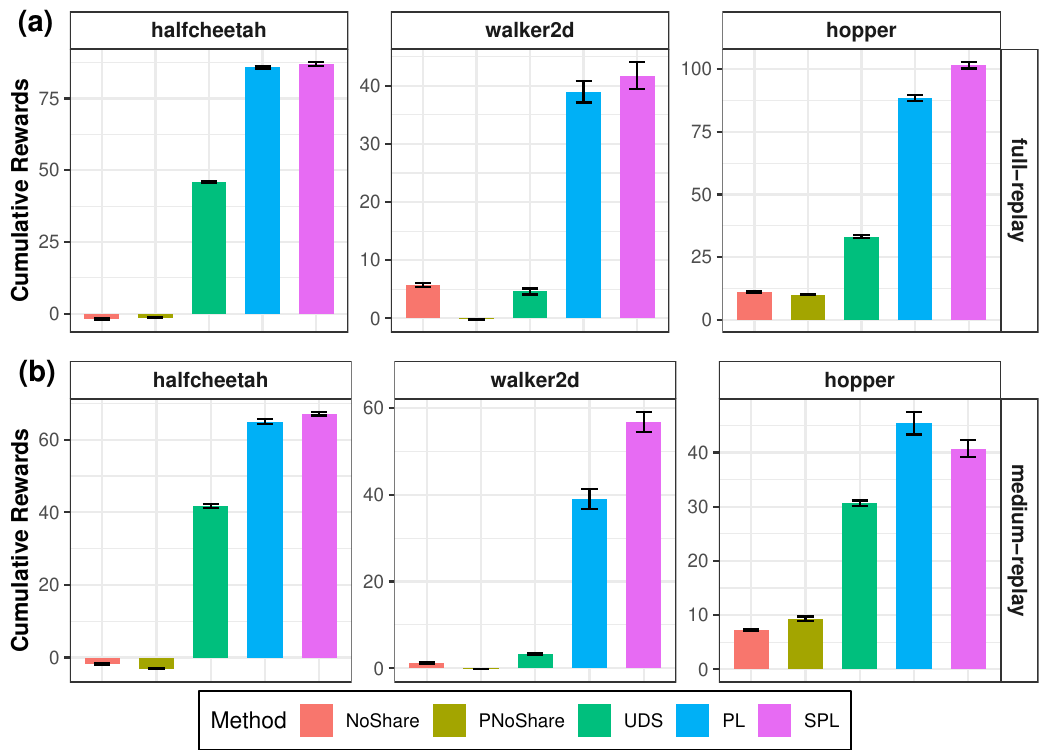}
\caption{Average cumulative reward of the policy learned by numerous model-based RL algorithms in the MuJoCo environments. The error bar shows the standard deviation. Rows represent different state-action distributions of $\ud$, and columns represent three different MuJoCo environments.}
\label{fig:mopo}
\end{figure}

Figure~\ref{fig:mopo}(a) reports the results for the full replay setting, and Figure~\ref{fig:mopo}(b) for the medium replay setting. We see that our proposed SPL generally achieves the largest cumulative reward in most cases, demonstrating its empirical effectiveness compared to the benchmark solutions. In particular, it surpasses PL, thanks to its integration of the semi-pessimism. Moreover, SPL substantially outperforms NoShare, PNoShare and UDS, highlighting the benefits of leveraging unlabeled data and accurately estimating the reward lower bound.

\subsection{Adaptive DBS Application}
\label{sec:real-data}

We revisit our motivation example of adaptive deep brain stimulation. DBS is an effective but difficult to optimize technology for Parkinson's disease. Offline RL offers a particularly useful framework for adaptive DBS, given the challenges of collecting patient data and the impracticality of certain actions. Meanwhile, adaptive DBS has its unique characteristics, such as the absence of reward information and the limited computing capability due to device constraints. Despite proof-of-principle studies, the optimization of adaptive DBS at scale is still in its infancy \citep{Neumann2023Adaptive}.

We analyze a semi-synthetic dataset collected by the University of California, San Francesco from a voluntary patient who received adaptive DBS under a pre-programed policy. The state information involves a five-dimensional processed power band data collected through the implanted electrodes. The action is a binary variable representing the stimulation amplitude: an action value of one corresponding to a high-amplitude stimulation at 3.1 milliamps, and a zero corresponding to a low-amplitude stimulation at 2.3 milliamps. The reward is defined as the negative log-transformed Bradykinesia score, which quantifies the severity of movement slowness, and always takes a negative value. The state information was collected every second, and the reward information every two minutes. We first fitted a random forests model for the observed reward given the state and action information, and another transformer model for the next state given the previous state and action information. Based on these models, we then simulated the state and reward data under a random policy of actions, which uniformly assigns one of the two stimulation amplitudes and each amplitude lasts for ten minutes. The total length of the random treatment stays for five hours and twenty minutes. The resulting data contains $n_\ld = 160$ observations of labeled reward, and $n_\ud = 19,200$ observations of unlabeled reward. 

We applied the proposed SPL method, along with the alternative solutions, to this data. We illustrate with the model-free Algorithm \ref{alg:mf-sorl} thanks to its simplicity, and as before, all methods employ FQI as the base RL algorithm for policy learning. The evaluation criterion is the cumulative reward, which takes a negative value in our setting, and a larger reward implies a better policy. We repeat the data simulation and evaluation 200 times. Table \ref{tab:dbs-result} reports the average cumulative reward along with the standard error in the parenthesis. We see that our SPL method achieves the best performance. In particular, SPL outperforms PDS, because of the random behavior policy and the large size of the unlabeled data. As a result, the conservative estimation of the Q-function as in PDS becomes unnecessary. Additionally, SPL outperforms PL, as the labeled data has a limited size, with fewer than 200 observations, which restricts its coverage. UDS performs poorly due to its large reward estimation errors. Finally, the methods using both labeled and unlabeled data, except for UDS, achieve higher cumulative rewards than those using only labeled data, demonstrating the practical value of semi-supervised RL.

\begin{table}[t!]
\centering
\caption{The average cumulative reward and the standard error (in the parenthesis) out of 200 replications for the adaptive DBS example. The larger the reward, the better the performance of the RL policy. }
\resizebox{\linewidth}{!}{
 \begin{tabular}{cccccc} \toprule
NoShare & PNoShare & PL & UDS & PDS & SPL \\ \midrule
-1198 (0.252) & -1179 (1.72) & -1166 (2.29) & -1205 (0.273) & -1142 (0.917) & \textbf{-1124} (0.816) \\ \bottomrule
\end{tabular}}
\label{tab:dbs-result}
\end{table}

\section{Discussions}
\label{sec:discussions}

In this article, we have developed a semi-supervised RL approach, which leverages both labeled and unlabeled data to tackle the challenges of distributional shift and missing reward. Our proposal is built upon the semi-pessimistic principle and a tight lower bound for the reward function. It relies on the core concept of semi-coverage, and the likelihood of this condition being satisfied increases with the amount of unlabeled data. 

Meanwhile, it is possible to further relax this semi-coverage condition to partial coverage, at the cost of a more complex uncertainty quantification of both the reward and state transition  functions. Specifically, we define a more pessimistic reward function,
\begin{equation*}
\widehat\rfun_\ppl(s, a) = \widehat\rfun_\spl(s, a) - z_{1-\alpha/2}\Delta_\taw(s, a),
\end{equation*}
which penalizes the pessimistic reward $\widehat\rfun_\spl$ in \eqref{eq:pessimistic-reward} by an additional transition-aware uncertainty quantification term, $\Delta_\taw(s, a) = \gamma \vub \sum_{s'} \big| \widehat{\tfun}(s'|s, a) - \tfun(s'|s, a) \big|$, that quantifies the uncertainty of the estimated state transition function, and $V_{\max}$ is as defined in Assumption \ref{con:value-func}. It essentially adopts the full pessimistic principle, and thus we refer to this method as \emph{pessimistic pseudo labeling} (PPL). The next theorem provides a theoretical justification for the model-based PPL. 

\begin{theorem}[Regret of the model-based PPL]\label{thm:mb-ta}
Suppose Assumptions~\ref{con:coverage}(a), \ref{con:uncertainty}, and \ref{con:value-func} hold. Let $B^*_{\mathcal{D}} \equiv \sup_{s, a} d^{\pi*}(s, a) / d_{\ld\cup\ud}(s, a)$, and $c_3$ denote some positive constant. Then the regret of the optimal policy $\widehat{\pi}$ computed from the model-based PPL,$ \E[J(\pi^*) - J(\widehat{\pi})]$, is upper bounded by
\begin{eqnarray}\label{eqn:regret-model-based-ta}
\begin{split}
\frac{\alpha (2R_{\max}+\gamma |\mathcal{S}|V_{\max})}{(1-\gamma)^2} + 
c_3\left[\frac{\sqrt{B_{\ld}^*}}{(1-\gamma)^2} \|\rfun-\widehat{\rfun}^\textup{TA}_{\ell}\|_{d_{\ld}} +\frac{\gamma V_{\max} \sqrt{B_{\mathcal{D}}^*}}{(1-\gamma)^2}\|\tfun-\widehat{\tfun}\|_{d_{\ld \cup \ud}} \right].
\end{split}
\end{eqnarray}
\end{theorem}

\noindent
We next compare this extended method PPL with our proposed SPL. Analytically, SPL adopts the semi-pessimistic principle that only quantifies the uncertainty of the estimated reward function, while PPL adopts the full pessimistic principle that accounts for the uncertainty in both the reward and transition dynamics. As such, SPL requires the semi-coverage condition, while PPL only requires the partial coverage condition. However, PPL comes at the cost of a more challenging task of estimating $\Delta_\taw$, which involves quantifying the uncertainty of the unknown transition probability $\tfun$. Theoretically, compared to the upper bound in \eqref{eqn:regret-model-based} in Theorem~\ref{thm:mb}, the upper bound in \eqref{eqn:regret-model-based-ta} in Theorem \ref{thm:mb-ta} involves the single-policy concentration coefficient $\mathcal{B}_{\mathcal{D}}^*$, instead of the uniform concentration coefficient $\mathcal{B}_{\mathcal{D}}$. It thus in effect relaxes Assumption \ref{con:coverage}(b). Furthermore, compared to the model-based offline RL algorithms \citep[e.g., ][]{kidambi2020morel} that only utilize labeled data $\ld$, the second term in the bracket of \eqref{eqn:regret-model-based-ta} is often substantially smaller, again reflecting the benefit of using both $\ld$ and $\ud$ to learn $\widehat\tfun$. Empirically, there is a trade-off between the two methods, as illustrated in Figure \ref{fig:tradeoff}: SPL performs better when the unlabeled data is generated by sub-optimal policies close to random exploration, whereas PPL excels when the unlabeled data distribution is closer to the optimal policy. 

\begin{figure}[H]
\centering
\includegraphics[width=0.5\linewidth]{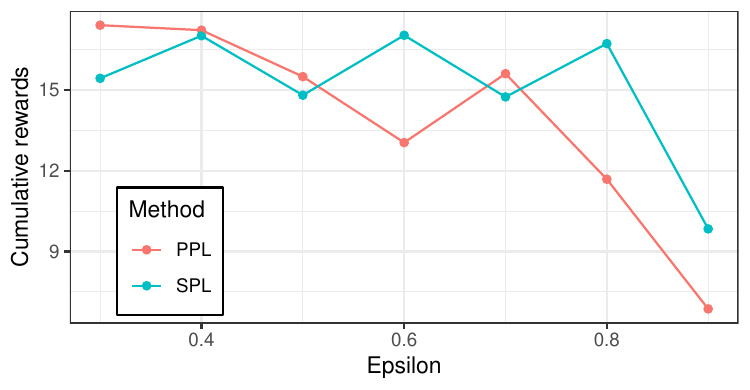}
\vspace*{-15pt}
\caption{The cumulative reward of SPL and PPL under different values of $\epsilon$, where the unlabeled data is generated using the $\epsilon$-greedy algorithm. A small $\epsilon$ indicates that the behavior policy that generates the unlabeled data is close to the optimal policy, whereas a large $\epsilon$ indicates that the behavior policy is close to random exploration.}
\label{fig:tradeoff}
\end{figure}

\bibliographystyle{Chicago}
\bibliography{ref-ssorl.bib}

\newpage
\appendix

\section{Proofs}
\label{sec:proofs}

\subsection{Notations}
\label{sec:proof-notation}

We first summarize the notations that are used throughout the proofs. Recall that $\ld$ and $\ud$ denote the labeled and unlabeled datasets, respectively. Let $n_{\ld}$ and $n_{\ud}$ denote their sample sizes. Let $\mathcal{D}$ denote their union $\mathcal{L}\cup \mathcal{U}$, and $n=n_{\ld}+n_{\ud}$. Let $d_{\ld}$, $d_{\ud}$ and $d_{\mathcal{D}}$ denote the state-action distributions of $\ld$, $\ud$, and $\mathcal{D}$, respectively. For any function $f$, let $\|f\|_{d_\mathcal{D}}$ denote the norm $\sqrt{\mathbb{E}_{(S,A)\sim d_{\mathcal{D}}} f^2(S,A)}$. 

For any deterministic policy $\pi$, i.e., for any state $s$, there exists a unique action $a$ such that $\pi(a|s)=1$, we use $\pi(s)$ to denote this action. Note that any estimated greedy policy $\widehat{\pi}$ derived from an estimated optimal Q-function $\widehat{Q}$ is deterministic, and the optimal policy $\pi^*$ is greedy with respect to the oracle optimal Q-function $Q^*$. 

Recall that $\rfun$ and $\tfun$ denote the oracle reward and state transition function, respectively. Let $\widehat{\rfun}_{\ell}$  and $\widehat{\tfun}$ denote the estimated lower bound of $\rfun$ and the estimated $\tfun$, respectively. Additionally, let $\event(s, a)$ denote the event $\{ |\widehat{\rfun}_{\ell}(s, a) - \rfun(s, a)| \leq \Gamma_{\alpha}(s, a) \}$. Define $\|\tfun-\widehat{\tfun}\|_{d_{\mathcal{D}}} \equiv \sqrt{\E\big[\E_{(S,A)\sim d_{\mathcal{D}}} \textup{TV}_{\widehat{\tfun}}^2(S,A)\big]}$ where the outer expectation $\mathbb{E}[\cdot]$ is taken over the estimated transition $\widehat{\tfun}$, and $\textup{TV}_{\widehat{\tfun}}(s,a)$ is the total variation distance between $\widehat{\tfun}(\cdot|s,a)$ and $\tfun(\cdot|s,a)$, i.e., $\sum_{s'}|\widehat{\tfun}(s'|s,a)-\tfun(s'|s,a)|/2$.

Let $\rho_0$ denote the probability mass function of a given reference initial state distribution. Let $\prob^{\pi}_{t}(s'|s) \equiv \prob^{\pi}(S_t=s'|S_0=s)$ denote the $t$-step transition probability from $s$ to $s'$ under $\pi$, and $d^{\pi}(s'|s) \equiv (1-\gamma) \sum\limits_{t=0}^{\infty} \gamma^t \prob^{\pi}_{t}(s'|s)$ denote the $\gamma$-discounted visitation probability for visiting $s'$ from a given initial state $s$ with $\prob^{\pi}_{0}(s'|s)=\mathbb{I}(s'=s)$. Moreover, let $d^{\pi}(s') \equiv \sum_s d^{\pi}(s'|s) \rho_0(s)$ denote the resulting $\gamma$-discounted visitation probability assuming the initial state $s$ follows $\rho_0$, and $d^{\pi}_0(s') \equiv \sum_s d^{\pi}(s'|s) \prob(S_0=s)$ as the corresponding distribution under the true initial state distribution.

Let $\mathcal{B}^*$ denote the Bellman optimality operator, such that $\mathcal{B}^*Q(s,a)=\rfun(s,a) + \gamma \sum\limits_{s'} \max\limits_{a'}$ $Q(s',a') \times \tfun(s'|s,a)$, and $\mathcal{B}_{\ell}^*$ denote the semi-pessimistic Bellman optimality operator, such that $\mathcal{B}_{\ell}^*Q(s,a)=\widehat{\rfun}_{\ell}(s,a)+\gamma \sum\limits_{s'} \max\limits_{a'} Q(s',a')\tfun(s'|s,a)$. 

Finally, recall that $g(s,a)$ denotes the linear basis function used to fit the reward function. Let $\Omega_{\ld}$ and $\Omega_{\ud}$ denote the expected Gram matrices $\mathbb{E}_{(S,A)\sim d_\ld} [g(S,A)g^\top(S,A)]$ and $\mathbb{E}_{(S,A)\sim d_\ud} [g(S,A)g^\top(S,A)]$, respectively. Define $\theta^*_{\text{\ini}}=\Omega_{\ld}^{-1}\mathbb{E} [g(S,A)R(S,A)]$,  $\theta^*_{\ld}=\Omega_{\ld}^{-1}\mathbb{E} [g(S,A)(R(S,A)-f(S,A))]$, and $\theta^*_{\ud}=\Omega_{\ud}^{-1}\mathbb{E} [g(S,A)f(S,A)]$.

\subsection{Proof of Lemma \ref{lemma:ppi-adv-formal}}
\begin{proof}
By definition, we have
\begin{eqnarray*}
\Delta_{\text{\ini}}(s,a)&=&\frac{1}{n_{\ld}}g^\top(s,a) \Omega_{\ld}^{-1} \mathbb{E}_{(S,A)\sim d_\ld} \Big\{g(S,A) [R-g^\top(S,A)\theta^*_{\text{\ini}}]^2 g^\top(S,A)\Big\}\Omega_{\ld}^{-1} g(s,a), \\
\Delta_{\text{\sug}}^*(s,a)&=&\frac{1}{n_{\ld}}g^\top(s,a)\Omega_{\ld}^{-1}\mathbb{E}_{(S,A)\sim d_\ld} \Big\{g(S,A)[R-f(S,A)-g^\top(S,A)\theta_{\ld}^*]^2g^\top(S,A)\Big\} \Omega_{\ld}^{-1}g(s,a)\\ 
    &+&\frac{1}{n_{\ud}}g^\top(s,a)\Omega_{\ld}^{-1}\mathbb{E}_{(S,A)\sim d_\ld} \Big\{g(S,A)[f(S,A)-g^\top(S,A)\theta_{\ud}^*]^2g^\top(S,A)\Big\} \Omega_{\ld}^{-1}g(s,a).
\end{eqnarray*}
As $n_{\ld}/n_{\ud}\to 0$, it suffices to show
\begin{eqnarray*}
&&g^\top(s,a) \Omega_{\ld}^{-1} \mathbb{E}_{(S,A)\sim d_\ld} \Big\{g(S,A) [R-g^\top(S,A)\theta^*_{\text{\ini}}]^2 g^\top(S,A)\Big\}\Omega_{\ld}^{-1} g(s,a)\\
&<&g^\top(s,a)\Omega_{\ld}^{-1}\mathbb{E}_{(S,A)\sim d_\ld} \Big\{g(S,A)[R-f(S,A)-g^\top(S,A)\theta_{\ld}^*]^2g^\top(S,A)\Big\} \Omega_{\ld}^{-1}g(s,a),
\end{eqnarray*}
or equivalently,
\begin{eqnarray*}
    &&\mathbb{E}_{(S,A)\sim d_\ld} \Big\{g(S,A) [R-g^\top(S,A)\theta^*_{\text{\ini}}]^2 g^\top(S,A)\Big\}\\
    &\prec& \mathbb{E}_{(S,A)\sim d_\ld} \Big\{g(S,A)[R-f(S,A)-g^\top(S,A)\theta_{\ld}^*]^2g^\top(S,A)\Big\}.
\end{eqnarray*}
As the conditional mean of the residual $R-R(S,A)$ is zero given $S$ and $A$, it suffices to show that
\begin{eqnarray}\label{eqn:prooflemmappiadv}
\begin{split}
    &\mathbb{E}_{(S,A)\sim d_\ld} \Big\{g(S,A) [R(S,A)-g^\top(S,A)\theta^*_{\text{\ini}}]^2 g^\top(S,A)\Big\}\\
    \prec & \mathbb{E}_{(S,A)\sim d_\ld} \Big\{g(S,A)[R(S,A)-f(S,A)-g^\top(S,A)\theta_{\ld}^*]^2g^\top(S,A)\Big\}. 
\end{split}
\end{eqnarray}
Specifically, under Assumption \ref{assump:auxmodel}, we have $\theta_\ld^*=(1-\beta_0)\Omega_{\ld}^{-1}\mathbb{E} [g(S,A)R(S,A)]-\beta_1=(1-\beta_0) \theta^*_{\text{\ini}}-\beta_1$. 
Therefore the right-hand-side of \eqref{eqn:prooflemmappiadv} equals
\begin{eqnarray*}
    &&\mathbb{E}_{(S,A)\sim d_\ld} \Big\{g(S,A)g^\top(S,A)e^2\Big\}\\
    &+&(1-\beta_0)^2 \mathbb{E}_{(S,A)\sim d_\ld} \Big\{g(S,A)g^\top(S,A)[R(S,A)-g^\top(S,A)\theta^*_{\text{\ini}}]^2\Big\}\\
    &-&2(1-\beta_0) \mathbb{E}_{(S,A)\sim d_\ld} \Big\{g(S,A)g^\top(S,A) e [R(S,A)-g^\top(S,A)\theta^*_{\text{\ini}}] \Big\}.
\end{eqnarray*}
Note that the last line is zero as $\mathbb{E} [e|R(S,A), g(S,A)]=0$. Then by Assumption \eqref{eqn:someinequality}, we obtain \eqref{eqn:prooflemmappiadv}, which completes the proof of Lemma \ref{lemma:ppi-adv-formal}.
\end{proof}

\subsection{An auxiliary lemma}

We introduce a supporting lemma that upper bounds the statistical estimation error encountered during each iteration of Algorithm \ref{alg:mf-sorl}, which in turn facilitates the proof of Theorem~\ref{thm:mf-fqi}. In the model-free algorithm, the Q-estimator $\widehat{Q}$ is computed recursively through FQI. We thus let $\widehat{Q}^{(k)}$ denote the Q-estimator at the $k$th iteration, and $\widehat{\pi}^{(k)}$ denote the estimated optimal policy derived from $\widehat{Q}^{(k)}$.

\newcommand{\floor}[1]{\lfloor #1 \rfloor}
\newcommand{\Var}{\textup{Var}}
\newcommand{\cQ}{\mathcal{Q}}
\newcommand{\cF}{\mathcal{F}}

\begin{lemma}\label{lemma:sample-optimal-gap}
Let $Q^{(k)} \equiv \argmin\limits_{Q \in \fclass} \| Q - \oblo \widehat{Q}^{(k-1)}\|^2_{d_{\mathcal{D}}}$, and $\epsilon^{(k)} \equiv \| \widehat{Q}^{(k)} - \oblo \widehat{Q}^{(k-1)}\|^2_{d_{\mathcal{D}}}$. Suppose Assumptions~\ref{con:value-func} to \ref{con:completeness} hold. Then, for any $k$, with probability at least $1-\left(\delta+\floor{\frac{n}{q}}\beta(q)\right)$,
\begin{equation}\label{eq:epsk}
\epsilon^{(k)}\le \frac{(784q-384)V_{\max}^2\ln(2\delta^{-1}|\cQ|^2|\cF|)+24q V_{\max}^2}{n}.
\end{equation}
\end{lemma}

\begin{proof}
For any $Q,Q'\in \mathcal{Q}$ and $f\in \mathcal{F}$, define the empirical loss function
\begin{eqnarray*}
        \|Q-\mathcal{B}_f Q'\|_{\mathcal{D}}^2\equiv \frac{1}{|\mathcal{D}|} \sum_{(S,A,S')\in \mathcal{D}} |Q(S,A)-f(S,A)-\gamma \max_a Q'(S',a)|^2.
\end{eqnarray*}
Note that
\begin{align}\label{eqn:maininequality}
\begin{split}
            \epsilon^{(k)} -\| Q^{(k)}
            - \oblo \widehat{Q}^{(k-1)}\|^2_{d_{\mathcal{D}}}
            =&
            \Big\{ \| \widehat{Q}^{(k)} - \mathcal{B}^*_{\ell} \widehat{Q}^{(k-1)}\|^2_{\mathcal{D}} - \| Q^{(k)} - \mathcal{B}^*_{\ell} \widehat{Q}^{(k-1)}\|^2_{\mathcal{D}} \Big\}
            \\
            &+ \Big\{ \| \widehat{Q}^{(k)} - \oblo \widehat{Q}^{(k-1)}\|^2_{d_{\mathcal{D}}} -  \| Q^{(k)}
            - \oblo \widehat{Q}^{(k-1)}\|^2_{d_{\mathcal{D}}}
            \\
            &\qquad- [\| \widehat{Q}^{(k)} - \mathcal{B}^*_{\ell} \widehat{Q}^{(k-1)}\|^2_{\mathcal{D}} - \| Q^{(k)} - \mathcal{B}^*_{\ell} \widehat{Q}^{(k-1)}\|^2_{\mathcal{D}}] \Big\},
            \\
            \leq& \| \widehat{Q}^{(k)} - \oblo \widehat{Q}^{(k-1)}\|^2_{d_{\mathcal{D}}} - \| Q^{(k)}
            - \oblo \widehat{Q}^{(k-1)}\|^2_{d_{\mathcal{D}}}
            \\
            &- [ \| \widehat{Q}^{(k)} - \mathcal{B}^*_{\ell} \widehat{Q}^{(k-1)}\|^2_{\mathcal{D}} - \| Q^{(k)} - \mathcal{B}^*_{\ell} \widehat{Q}^{(k-1)}\|^2_{\mathcal{D}} ]
\end{split}
\end{align}
where the inequality holds by the definition of $\mathcal{Q}^{(k)}$ and the fact that $Q^{(k)}\in \mathcal{Q}$. The right-hand-side of \eqref{eqn:maininequality} can be further upper bounded by the following zero-mean empirical process:
\begin{equation}\label{eqn:emp}
        \sup_{f \in \mathcal{F}, Q, Q' \in \fclass} \frac{1}{|\mathcal{D}|} \sum_{(S,A,S')\in \mathcal{D}} \psi(S, A, S'; f, Q, Q') - \E[\psi(S, A, S'; f, Q, Q')],
\end{equation}
where $\psi(s, a, s'; f, Q, Q') = |Q(s, a) - f(s,a) - \gamma \max_a Q'(s', a)|^2 - |Q^{(k)}(s, a) - f(s,a) - \gamma \max_a Q'(s', a)|^2$.

Suppose the dataset $\mathcal{D}$ contains a total of $K$ trajectories, with the $j$th trajectory containing $H_j$ many data points. We can sort all state-action-next-state triplets in $\mathcal{D}$ as 
\begin{eqnarray}\label{eqn:someset}
\big\{O_{1, 0}, \ldots, O_{1, H_1}, \ldots, O_{K, 0}, \ldots, O_{K, H_K} \big\},
\end{eqnarray}
where each $O_{j, t}$ denotes a triplet at $t$th time step in the $j$th trajectory. For the $i$th triplet in \eqref{eqn:someset}, we denote $\psi(S_i, A_i, S'_i; f, Q, Q')$ as $\psi_i$, which leads to the sequence $\Psi=(\psi_1, \psi_2, \dots, \psi_n)$. We partition $\Psi$ into subsequences of length $q$. Let $U_i=(\psi_{iq+1}, \dots, \psi_{(i +1)q})$. Then $\Psi$ can be represented by $(U_0, \dots, U_{\floor{\frac{n}{q}}-1}, \psi_{\floor{\frac{n}{q}}q+1}, \dots, \psi_n)$. 
    
By Assumption \ref{con:indep-tuple}, following the discussion on Lemma 4.1 in \citet{dedecker2002maximal}, we can construct a sequence of random variables $\{\tilde{U}_i\}_i=\{(\tilde{\psi}_{i q+1}, \dots, \tilde{\psi}_{(i+1)q})\}_i$, such that
\begin{enumerate}
\item[(i)] Each $U_i$ has the same marginal distribution as $\tilde{U}_i$; 
\item[(ii)] For each $i$, $\mathbb{P}(U_i=\tilde{U}_i)\ge 1-\beta(q)$;
\item[(iii)] $\tilde{\Psi}=(\tilde{U}_0, \dots, \tilde{U}_{\floor{\frac{n}{q}}-1}, \psi_{\floor{\frac{n}{q}}q+1}, \dots, \psi_n)$ equals $\Psi$ with probability at least $1-\floor{\frac{n}{q}}\beta(q)$;
\item[(iv)] $\{\tilde{U}_{2i}\}_{i\geq0}$ are i.i.d., and so are $\{\tilde{U}_{2i+1}\}_{i\geq0}$.
\end{enumerate}
Let $\omega_i=\tilde{\psi}_{i q+1}+\dots +\tilde{\psi}_{(i+1)q}$, the sum of all elements in $\tilde{U}_i$. It follows from (iv) that both $\{\omega_{2i}\}_{i\geq0}$ and $\{\omega_{2i+1}\}_{i\geq0}$ are i.i.d. 

Without loss of generality, suppose all functions in $\mathcal{F}$ are uniformly bounded by $R_{\max}$. When this condition does not hold, we can restrict to a subclass of $\mathcal{F}$ bounded by $R_{\max}$. Additionally, by Assumption \ref{con:value-func}, $|Q^{(k)}(s,a)|$ is uniformly upper bounded by $V_{\max}$. It follows that $|\tilde{\psi}_i|\leq 8V_{\max}^2$, and $|\omega_i|\leq 8q V_{\max}^2$. Besides, with some calculations, the variance of $\psi_i$ can be upper bounded by
\begin{eqnarray}\label{eqn:variancebound}
\begin{split}
            \Var(\psi_i) \leq 32V_{\max}^2 \epsilon^{(k)} \equiv \sigma^2_{\psi}.
\end{split}
\end{eqnarray}
Therefore, we have
\begin{align*}
\Var(\omega_0) =  \sum_{i=1}^q \Var(\tilde{\psi}_i)+2\sum_{1\le i<j\le q}cov(\tilde{\psi}_i,\tilde{\psi}_j) 
\le q \sigma_{\psi}^2+2q(q-1)\sigma_{\psi}^2 = q(2q-1)\sigma_{\psi}^2 :=\sigma^2,
\end{align*}
as $\tilde{\psi_i}$ has the same marginal distribution to $\psi_i$.

Since all the function classes are finite hypothesis classes, we apply the Bernstein's inequality to upper bound \eqref{eqn:emp}. Let $n_1$ denote the number of $\omega_i$'s with odd $i$, and $n_2$ denote the number with even $i$. We have $n_1+n_2=\floor{\frac{n}{q}}$. It follows from the Bernstein's inequality that, for any $\epsilon_1, \epsilon_2 \geq 0$,
\begin{align*}
\mathbb{P}\left(\sum_{i=0}^{n_1-1} \left[\omega_{2i+1} - \E(\omega_{2i+1})\right]>\epsilon_1\right)        & \le \exp\left(\frac{-\epsilon_1^2}{2(n_1\sigma^2+8q V^2_{\max}\epsilon_1/3)}\right), \\
\text{and }\mathbb{P}\left(\sum_{i=0}^{n_2-1} \left[\omega_{2i} - \E(\omega_{2i})\right]>\epsilon_2\right) & \le \exp\left(\frac{-\epsilon_2^2}{2(n_2\sigma^2+8q V^2_{\max}\epsilon_2/3)}\right).
\end{align*}
This above inequalities hold for each $f\in \mathcal{F}$ and $Q,Q'\in \mathcal{Q}$. Then, for any $\delta_1,\delta_2\ge 0$, we have that, with probability at least $1-\left(\delta_1+\delta_2+\floor{\frac{n}{q}}\beta(q)\right)$,
\begin{align*}
        \sum_{i=1}^{\floor{\frac{n}{q}}q}\psi_i-\floor{\frac{n}{q}}q\E(\psi)
        \leq & \frac{8 q V_{\max}^2}{3}\left[\ln\left(\frac{|\cQ|^2|\cF|}{\delta_1}\right)+\ln\left(\frac{|\cQ|^2|\cF|}{\delta_2}\right)\right]
        \\
             & +\sqrt{2n_1\sigma^2\ln\left(\frac{|\cQ|^2|\cF|}{\delta_1}\right)}+\sqrt{2n_2\sigma^2\ln\left(\frac{|\cQ|^2|\cF|}{\delta_2}\right)}.
\end{align*}
As such, with probability at least $1-\left(\delta_1+\delta_2+\floor{\frac{n}{q}}\beta(q)\right)$, \eqref{eqn:emp} can be bounded as:
\begin{align*}
\frac{1}{n}\sum_{i=1}^n \psi_i-\E(\psi_i)
=   & \frac{\floor{\frac{n}{q}}q}{n}\frac{1}{\floor{\frac{n}{q}}q}\left(\sum_{i=1}^{\floor{\frac{n}{q}}q}\psi_i-\floor{\frac{n}{q}}q\E(\psi_i)\right)+ \frac{1}{n}\left(\sum_{i=\floor{\frac{n}{q}}+1}^n \psi_i\right) \\
\le & \frac{8q V_{\max}^2}{3n}\left[\ln\left(\frac{|\cQ|^2|\cF|}{\delta_1}\right)+\ln\left(\frac{|\cQ|^2|\cF|}{\delta_2}\right)\right] \\
& +\frac{1}{n}\left[\sqrt{2n_1\sigma^2\ln\left(\frac{|\cQ|^2|\cF|}{\delta_1}\right)}+\sqrt{2n_2\sigma^2\ln\left(\frac{|\cQ|^2|\cF|}{\delta_2}\right)}\right]+\frac{8q V_{\max}^2}{n}  \\
\le & \frac{8q V_{\max}^2}{3n}\left[\ln\left(\frac{|\cQ|^2|\cF|}{\delta_1}\right)+\ln\left(\frac{|\cQ|^2|\cF|}{\delta_2}\right)\right]  \\
& +\sqrt{\frac{128(2q-1)V_{\max}^2 \epsilon^{(k)}}{n}\left[\ln\left(\frac{|\cQ|^2|\cF|}{\delta_1}\right)+\ln\left(\frac{|\cQ|^2|\cF|}{\delta_2}\right)\right]}+\frac{8q V_{\max}^2}{n}.
\end{align*}
where the inequality is due to the Cauchy-Schwarz inequality. Letting $\delta_1=\delta_2={\delta}/{2}$, the last term equals 
\begin{align*}
& \sqrt{\frac{256(2q-1)V_{\max}^2 \epsilon^{(k)} \ln(2\delta^{-1}|\cQ|^2|\cF|)}{n}}+\frac{16q V_{\max}^2\ln(2\delta^{-1}|\cQ|^2|\cF|)+24q V_{\max}^2}{3N} \\
\le  & \frac{\epsilon^{(k)}}{2}+\frac{(784q - 384)V_{\max}^2\ln(2\delta^{-1}|\cQ|^2|\cF|)+24q V_{\max}^2}{3N},
\end{align*}
where the last inequality again follows from the Cauchy-Schwarz inequality. 
    
Therefore, with probability at least $1-\left(\delta+\floor{\frac{n}{q}}\beta(q)\right)$,
\begin{eqnarray}\label{eqn:varepsilon1k}
\epsilon^{(k)}-\| Q^{(k)}-
            \oblo \widehat{Q}^{(k-1)}\|^2_{d_{\mathcal{D}}} \le \frac{\epsilon^{(k)}}{2}+\frac{(784q-384)V_{\max}^2\ln(2\delta^{-1}|\cQ|^2|\cF|)+24q V_{\max}^2}{3n}.
\end{eqnarray}
By completeness, we have $\| Q^{(k)}- \oblo \widehat{Q}^{(k-1)}\|_{d_{\mathcal{D}}}=0$. Consequently, 
    \begin{eqnarray*}
        \epsilon^{(k)}\le \frac{(784q-384)V_{\max}^2\ln(2\delta^{-1}|\cQ|^2|\cF|)+24q V_{\max}^2}{n},
    \end{eqnarray*}
    with probability at least $1-\left(\delta+\floor{\frac{n}{q}}\beta(q)\right)$. This completes the proof of Lemma \ref{lemma:sample-optimal-gap}. 
\end{proof}

\subsection{Proof of Theorem~\ref{thm:mf-fqi}}
\label{sec:prooffqi}

\begin{proof}
By the performance difference lemma in \cite{kakade2002approximately}, 
\begin{align*}
        & J(\pi^*) - J(\widehat{\pi}^{(K)}) = \frac{1}{1-\gamma} \E_{S \sim d^{\widehat{\pi}_K}}  [Q^*(S, \pi^*(S)) - Q^*(S, \widehat{\pi}(S))]
        \\
        = \;   & \frac{1}{1-\gamma} \E_{S \sim d^{\widehat{\pi}_K}}  [Q^*(S, \pi^*(S)) - \widehat{Q}^{(K)}(S, \pi^*(S))
        + \widehat{Q}^{(K)}(S, \pi^*(S)) - Q^*(S, \widehat{\pi}^{(K)}(S))]
        \\
        \leq \; & \frac{1}{1-\gamma} \E_{S \sim d^{\widehat{\pi}_{K}}}  [Q^*(S, \pi^*(S)) - \widehat{Q}^{(K)}(S, \pi^*(S)) + \widehat{Q}^{(K)}(S, \widehat{\pi}^{(K)}(S)) - Q^*(S, \widehat{\pi}^{(K)}(S))].
\end{align*}
For any $s \in \sspace$, the term $\widehat{Q}^{(K)}(s, \widehat{\pi}^{(K)}(s)) - Q^*(s, \widehat{\pi}^{(K)}(s))$ satisfies that,
\begin{align*}
             &
        \widehat{Q}^{(K)}(s, \widehat{\pi}^{(K)}(s)) - Q^*(s, \widehat{\pi}^{(K)}(s)) \\
        =  \;  & \widehat{Q}^{(K)}(s, \widehat{\pi}^{(K)}(s)) - \obo \widehat{Q}^{(K-1)}(s, \widehat{\pi}^{(K)}(s)) + \obo \widehat{Q}^{(K-1)}(s, \widehat{\pi}^{(K)}(s)) - \obo Q^*(s, \widehat{\pi}^{(K)}(s))
        \\
        = \;   & \widehat{Q}^{(K)}(s, \widehat{\pi}^{(K)}(s)) - \obo \widehat{Q}^{(K-1)}(s, \widehat{\pi}^{(K)}(s)) 
        \\
        &+ \gamma \sum_{s'} \tfun(s'|s, \widehat{\pi}^{(K)}(s)) [\widehat{Q}^{(K-1)}(s', \widehat{\pi}^{(K-1)}(s')) - Q^*(s', \pi^*(s'))]
        \\
        \leq \; & \widehat{Q}^{(K)}(s, \widehat{\pi}^{(K)}(s)) - \obo \widehat{Q}^{(K-1)}(s, \widehat{\pi}^{(K)}(s)) 
        \\
        &+ \gamma \sum_{s'} \tfun(s'|s, \widehat{\pi}^{(K)}(s)) [\widehat{Q}^{(K-1)}(s', \widehat{\pi}^{(K-1)}(s')) - Q^*(s', \widehat{\pi}^{(K-1)}(s'))],
\end{align*}
where the first equality is due to that $Q^* = \obo Q^{*}$, and the inequality arises from the fact that $Q^*(s, \pi^*(s)) \geq Q^*(s, \widehat{\pi}^{(K-1)}(s))$ for any $s \in \sspace$. Next, we recursively apply the above inequality to $\widehat{Q}^{(k)}-Q^*$ for $k=K-1,\cdots,1$, and obtain that
\begin{align}\nonumber
             &
        \widehat{Q}^{(K)}(s, \widehat{\pi}^{(K)}(s)) - Q^*(s, \widehat{\pi}^{(K)}(s))
        \\\nonumber
        \leq \; & \widehat{Q}^{(K)}(s, \widehat{\pi}^{(K)}(s)) - \obo \widehat{Q}^{(K-1)}(s, \widehat{\pi}^{(K)}(s))
        \\\nonumber
             & + \gamma \sum_{s'} \tfun(s'|s, \widehat{\pi}^{(K)}(s)) [\widehat{Q}^{(K-1)}(s', \widehat{\pi}^{(K-1)}(s')) - \obo \widehat{Q}^{(K-2)}(s', \widehat{\pi}^{(K-1)}(s')) ]
        \\\nonumber
             & + \gamma^2 \sum_{s''} \tfun(s''|s', \widehat{\pi}^{(K-1)}(s')) \sum_{s'} \tfun(s'|s, \widehat{\pi}^{(K)}(s)) [\widehat{Q}^{(K-2)}(s'', \widehat{\pi}^{(K-2)}(s'')) - Q^*(s'', \widehat{\pi}^{(K-2)}(s'')) ]
        \\\label{eqn:firstterm}
        \leq \; & \sum_{h=0}^{K-1} \sum_{s_h} \gamma^{h} \prob_h^{\widetilde{\pi}}(s_h|s) \Big[ \widehat{Q}^{(K-h)}(s_h, \widehat{\pi}^{(K-h)}(s_h)) - \obo \widehat{Q}^{(K-h-1)}(s_h, \widehat{\pi}^{(K-h)}(s_h)) \Big]      
        \\\label{eqn:secondterm}
             & + \gamma^K \sum_{s_K} \prob_K^{\widetilde{\pi}}(s_K|s) \Big[\widehat{Q}^{(0)}(s_K, \widehat{\pi}^{(0)}(s_K)) - Q^*(s_K, \widehat{\pi}^{(0)}(s_K)) \Big],
\end{align}
where $\widetilde{\pi}$ denotes a nonstationary policy that assigns actions according to $\widehat{\pi}^{(K-k)}$ at time $k$. We next upper bound \eqref{eqn:firstterm} and \eqref{eqn:secondterm}, respectively. 

For \eqref{eqn:firstterm}, for any $0\le k<K$, the difference between $\obo_{\ell} \widehat{Q}^{(k)}$ and $\obo \widehat{Q}^{(k)}$ is equal to $\widehat{\rfun}_{\ell}-\rfun$, and is upper bounded by $[\widehat{\rfun}_{\ell}(s,a)-\rfun(s,a)]\indicator(\event^c(s,a))$. As such, we can upper bound \eqref{eqn:firstterm} by
\begin{eqnarray*}
        &&\underbrace{\sum_{h=0}^{K-1} \sum_{s_h} \gamma^{h} \prob_h^{\widetilde{\pi}}(s_h|s) \Big[ \widehat{Q}^{(K-h)}(s_h, \widehat{\pi}^{(K-h)}(s_h)) - \obo_{\ell} \widehat{Q}^{(K-h-1)}(s_h, \widehat{\pi}^{(K-h)}(s_h)) \Big]}_{\zeta(s)}\\&+&\sum_{h=0}^{K-1} \sum_{s_h} \gamma^{h} \prob_h^{\widetilde{\pi}}(s_h|s) \left[\widehat{\rfun}_{\ell}(s_h,\widehat{\pi}^{(K-h-1)}(s_h))-\rfun(s_h,\widehat{\pi}^{(K-h-1)}(s_h)) \right]\indicator(\event^c(s_h,\widehat{\pi}^{(K-h-1)}(s_h))).
\end{eqnarray*}
Following similar arguments to bounding the first term on the right-hand-side of \eqref{eqn:onekeyequation}, we can show that the expected value of the second line is upper bounded by
\begin{eqnarray}\label{eqn:somesecondline}
        2\alpha R_{\max} \sum_{h=0}^{K-1}\gamma^h\le \frac{2\alpha R_{\max}}{1-\gamma}.
\end{eqnarray}
Additionally, following similar arguments in bounding the second term on the right-hand-side of \eqref{eqn:onekeyequation}, we can show that
\begin{align*}
        \E [\E_{S\sim d^{\widehat{\pi}}} \zeta(S)]
        \le & \frac{1}{1-\gamma}\sup_{\pi_1,\pi_2}\E \Big[\E_{\substack{S\sim d^{\pi_1}
        \\ S'\sim d^{\pi_2}(\cdot|S)}} \max_{k\ge 1} |\widehat{Q}^{(k)}(S', \pi_2(S'))-\obo_{\ell} \widehat{Q}^{(k-1)}(S',\pi_2(S'))|\Big]
        \\
        \le & \frac{1}{1-\gamma}\sup_{\pi_1,\pi_2}\Big\{\E \Big[\E_{\substack{S\sim d^{\pi_1}                                                                                                                     \\ S'\sim d^{\pi_2}(\cdot|S)}} \max_{k\ge 1} |\widehat{Q}^{(k)}(S', \pi_2(S'))-\obo_{\ell} \widehat{Q}^{(k-1)}(S',\pi_2(S'))|^2\Big]\Big\}^{1/2}
        \\
        \le & \frac{1}{\sqrt{\bar{c}}(1-\gamma)}\sup_{\pi}\Big\{\E \Big[\E_{\substack{S\sim \rho_0                                                                                                                \\ S'\sim d^{\pi}(\cdot|S)}} \max_{k\ge 1} |\widehat{Q}^{(k)}(S', \pi(S'))-\obo_{\ell} \widehat{Q}^{(k-1)}(S',\pi(S'))|^2\Big]\Big\}^{1/2}
        \\
        \le & \frac{\sqrt{B_{\mathcal{D}}}}{\sqrt{\bar{c}}(1-\gamma)}\Big\{\E \Big[\E_{(S,A)\sim d_{\mathcal{D}}} \max_{k\ge 1} |\widehat{Q}^{(k)}(S, A)-\obo_{\ell} \widehat{Q}^{(k-1)}(S,A)|^2\Big]\Big\}^{1/2}.
\end{align*}
Let $\event(\delta, q)$ denote the event that the inequality \eqref{eq:epsk} holds. The last line above can be decomposed into
\begin{eqnarray}\label{eqn:anotherfirstline}
        &&\frac{2\sqrt{B_{\mathcal{D}}}}{\sqrt{\bar{c}}(1-\gamma)}\Big\{\E \Big[\E_{(S,A)\sim d_{\mathcal{D}}} \max_{k\ge 1} |\widehat{Q}^{(k)}(S, A)-\obo_{\ell} \widehat{Q}^{(k-1)}(S,A)|^2\Big]\indicator(\event(\delta,q))\Big\}^{1/2}\\\label{eqn:anothersecondline}
        &+&\frac{2\sqrt{B_{\mathcal{D}}}}{\sqrt{\bar{c}}(1-\gamma)}\Big\{\E \Big[\E_{(S,A)\sim d_{\mathcal{D}}} \max_{k\ge 1} |\widehat{Q}^{(k)}(S, A)-\obo_{\ell} \widehat{Q}^{(k-1)}(S,A)|^2\Big]\indicator(\event^c(\delta,q))\Big\}^{1/2}.
\end{eqnarray}
By Lemma \ref{lemma:sample-optimal-gap}, \eqref{eqn:anotherfirstline} and \eqref{eqn:anothersecondline} can be upper bounded by
\begin{align*}
         O\left(\frac{\sqrt{B_{\mathcal{D}}}}{\sqrt{\bar{c}}(1-\gamma)} V_{\max}\sqrt{\frac{q}{n}}\sqrt{\ln \Big(\frac{|\mathcal{F}||\mathcal{Q}|}{\delta}\Big) }\right),
        \textup{ and } O\left(\frac{\sqrt{B_{\mathcal{D}}}}{\sqrt{\bar{c}}(1-\gamma)}V_{\max} \sqrt{\delta+\beta(q) n/q}\right). 
\end{align*}
Ssetting $\delta$ to $n^{-1}$ and $q$ to $C\ln(n)$ for some sufficiently large constant $C>0$, it follows from Assumption~\ref{con:indep-tuple} that,
\begin{equation}\label{eqn:somesecondline2}
        \begin{split}
            \E [\E_{S\sim d^{\widehat{\pi}}} \zeta(S)]
            =&O\left(\frac{\sqrt{B_{\mathcal{D}}}}{\sqrt{\bar{c}}(1-\gamma)} V_{\max}\left( \sqrt{\frac{q\ln (|\mathcal{F}||\mathcal{Q}|n) }{n}} \right)\right)
        \end{split}
\end{equation}

For \eqref{eqn:secondterm}, by Assumption \ref{con:value-func}, any Q-function is uniformly bounded by $V_{\max}$. Since $Q^{(0)}$ is a zero function, \eqref{eqn:secondterm} is upper bounded by $\gamma^K V_{\max}$. This, together with \eqref{eqn:somesecondline} and \eqref{eqn:somesecondline2}, yields that
\begin{eqnarray}\label{eqn:someimportantinequality1}
\begin{split}
            &\E \Big(\E_{S \sim d^{\widehat{\pi}_{K}}}  [\widehat{Q}^{(K)}(S, \widehat{\pi}^{(K)}(S)) - Q^*(S, \widehat{\pi}^{(K)}(S))]\Big)\\
            =&\gamma^K V_{\max}+\frac{2\alpha R_{\max}}{1-\gamma}+O\left(\frac{\sqrt{B_{\mathcal{D}}}V_{\max}}{\sqrt{\bar{c}}(1-\gamma)}\sqrt{\frac{q\ln (|\mathcal{F}||\mathcal{Q}|n) }{n}}\right)
\end{split}
\end{eqnarray}

Similarly, we have that,
\begin{align*}
             &
        Q^*(s, \pi^*(s)) - \widehat{Q}^{(K)}(s, \pi^*(s))
        \\
        =    & \obo Q^*(s, \pi^*(s)) - \obo \widehat{Q}^{(K-1)}(s, \pi^*(s)) + \obo \widehat{Q}^{(K-1)}(s, \pi^*(s)) - \widehat{Q}^{(K)}(s, \pi^*(s))
        \\
        =    & \gamma \sum_{s'} \tfun(s'|s, \pi^*(s)) \left[Q^*(s', \pi^*(s')) - \widehat{Q}^{(K-1)}(s', \widehat{\pi}^{(K-1)}(s')) \right] + \obo \widehat{Q}^{(K-1)}(s, \pi^*(s)) - \widehat{Q}^{(K)}(s, \pi^*(s))
        \\
        \leq & \gamma \sum_{s'} \tfun(s'|s, \pi^*(s)) \left[Q^*(s', \pi^*(s')) - \widehat{Q}^{(K-1)}(s', \pi^*(s')) \right] + \obo \widehat{Q}^{(K-1)}(s, \pi^*(s)) - \widehat{Q}^{(K)}(s, \pi^*(s)).
\end{align*}
Applying the above inequality $K-1$ times recursively, we obtain that,
\begin{align*}
        Q^*(s, \pi^*(s)) - \widehat{Q}^{(K)}(s, \pi^*(s))
        \leq &
        \sum_{h=0}^{K-1} \gamma^h \sum_{s_h}\prob_{h}^{\pi^*}(s_h|s) \left[\obo \widehat{Q}^{(K-h-1)}(s_h, \pi^*(s_h)) - \widehat{Q}^{(K-h)}(s_h, \pi^*(s_h)) \right]
        \\
             & + \gamma^K \sum_{s_K} \prob_{K}^{\pi^*}(s_K|s) \left[Q^*(s_K, \pi^*(s_K)) - \widehat{Q}^{(0)}(s_K, \pi^*(s_K)) \right].
\end{align*}
The second line above can be upper bounded by $\gamma^K \vub$. The first line above can be further decomposed into the sum of
\vspace{-0.01in}
\begin{eqnarray}\label{eqn:anotherveryimportantsecondline}
        \begin{split}
            &\underbrace{\sum_{h=0}^{K-1} \gamma^h \sum_{s_h}\prob_{h}^{\pi^*}(s_h|s) [\obo_{\ell} \widehat{Q}^{(K-h-1)}(s_h, \pi^*(s_h)) - \widehat{Q}^{(K-h)}(s_h, \pi^*(s_h))]}_{\zeta^*(s)}\\
            +&\sum_{h=0}^{K-1} \gamma^h \sum_{s_h}\prob_{h}^{\pi^*}(s_h|s) [\obo \widehat{Q}^{(K-h-1)}(s_h, \pi^*(s_h)) - \obo_{\ell} \widehat{Q}^{(K-h-1)}(s_h, \pi^*(s_h))].
        \end{split}
\end{eqnarray}

Following similar arguments to bound $\E [\E_{S\sim d^{\widehat{\pi}}} \zeta(S)]$, we can show that,
\begin{eqnarray*}\label{eqn:someothersecondline2}
        \E [\E_{S\sim d^{\widehat{\pi}}} \zeta^*(S)]=O\left(\frac{\sqrt{B_{\mathcal{D}}}V_{\max}}{\sqrt{\bar{c}}(1-\gamma)} \sqrt{\frac{q\ln (|\mathcal{F}||\mathcal{Q}|n) }{n}} \right).
\end{eqnarray*}
Meanwhile, the second term in \eqref{eqn:anotherveryimportantsecondline} can be upper bounded by $(1-\gamma)^{-1} \sum_{s'} d^{\pi^*}(s'|s)$ $|\rfun(s',\pi^*(s'))-\widehat{\rfun}_{\ell}(s',\pi^*(s'))|$. Using similar arguments to bounding $\E(I_2)$ in the proof of Theorem \ref{thm:mb}, it can be further upper bounded by $\bar{c}^{-1/2}(1-\gamma)^{-1}\sqrt{B_{\ld}^*\E_{(S,A)\sim d_{\mathcal{L}}} |\rfun(S,A)-\widehat{\rfun}_{\ell}(S,A)|^2}$. This, together with the above bounds, yields that,
\begin{eqnarray*}
        \E\left( \E_{S\sim d^{\widehat{\pi}}} [Q^*(S,\pi^*(S))-\widehat{Q}^{(K)}(S,\pi^*(S))] \right)=\gamma^K V_{\max}
        +O\left(\frac{\sqrt{B_{\ld}^*}}{\sqrt{\bar{c}}(1-\gamma)}\|\rfun-\widehat{\rfun}_{\ell}\|_{d_{\ld}}\right)\\+
        O\left(\frac{\sqrt{B_{\mathcal{D}}}V_{\max}}{\sqrt{\bar{c}}(1-\gamma)}  \sqrt{\frac{q\ln (|\mathcal{F}||\mathcal{Q}|n) }{n}} \right).
\end{eqnarray*}
This, together with \eqref{eqn:someimportantinequality1} and the performance difference lemma, yields the desired upper bound, which completes the proof of Theorem~\ref{thm:mf-fqi}.
\end{proof}

\subsection{Proof of Theorem~\ref{thm:mb}}
\begin{proof}
With the model-based method, the estimated optimal policy $\widehat{\pi}$ is uniquely determined by the estimated transition and reward functions $\widehat{\tfun}$ and $\widehat{\rfun}_{\ell}$. Let $\widehat{Q}$ denote the estimated optimal Q-function based on $\widehat{\tfun}$ and $\widehat{\rfun}_{\ell}$, and $\widehat{\pi}$ is given by the greedy policy with respect to $\widehat{Q}$. By the performance difference lemma in \citep{kakade2002approximately},
\begin{eqnarray}\label{eq:mb-decompo}
& & J(\pi^*) - J(\widehat{\pi}) = \frac{1}{1-\gamma} \E_{S \sim d_0^{\widehat{\pi}}}  [Q^*(S, \pi^*(S)) - Q^*(S, \widehat{\pi}(S))] \nonumber \\
& = & \frac{1}{1-\gamma} \E_{S \sim d_0^{\widehat{\pi}}}  [Q^*(S, \pi^*(S)) - \widehat{Q}(S, \pi^*(S)) + \widehat{Q}(S, \pi^*(S)) - Q^*(S, \widehat{\pi}(S))] \\
&  \leq & \frac{1}{1-\gamma} \E_{S \sim d_0^{\widehat{\pi}}}[\underbrace{\widehat{Q}(S, \widehat{\pi}(S)) - Q^*(S, \widehat{\pi}(S))}_{I_1}] 
+ \frac{1}{1-\gamma}\E_{S \sim d_0^{\widehat{\pi}}}  [\underbrace{Q^*(S, \pi^*(S)) - \widehat{Q}(S, \pi^*(S))}_{I_2} ], \nonumber
\end{eqnarray}
where the last inequality is due to that $\widehat{Q}(s, \widehat{\pi}(s)) \geq \widehat{Q}(s, \pi^*(s))$ for any $s \in \sspace$. We next upper bound $\E(I_1)$ and $\E(I_2)$, respectively. 

For $\E(I_1)$, with some calculations, we have that,
\begin{eqnarray}\label{eq:mb-term2}
& &\widehat{Q}(s, \widehat{\pi}(s)) - Q^*(s, \widehat{\pi}(s)) \\
& = &\widehat{\rfun}_{\ell}(s, \widehat{\pi}(s)) - \rfun(s, \widehat{\pi}(s))+ \gamma \sum_{s'}\widehat{\tfun}(s'|s, \widehat{\pi}(s))\widehat{Q}(s', \widehat{\pi}(s')) - \gamma \sum_{s'}\tfun(s'|s, \widehat{\pi}(s))Q^*(s', \pi^*(s')) \nonumber\\
& \leq & \widehat{\rfun}_{\ell}(s, \widehat{\pi}(s)) - \rfun(s, \widehat{\pi}(s)) + \gamma \sum_{s'}\widehat{\tfun}(s'|s, \widehat{\pi}(s))\widehat{Q}(s', \widehat{\pi}(s')) - \gamma \sum_{s'}\tfun(s'|s, \widehat{\pi}(s))Q^*(s', \widehat{\pi}(s')) \nonumber \\
& = & [\widehat{\rfun}_{\ell}(s, \widehat{\pi}(s)) - \rfun(s, \widehat{\pi}(s))] \indicator(\event(s, \widehat{\pi}(s)) 
+ [\widehat{\rfun}_{\ell}(s, \widehat{\pi}(s)) - \rfun(s, \widehat{\pi}(s))] \indicator(\event^c(s, \widehat{\pi}(s))) \nonumber \\
& & +  \gamma \sum_{s'}[\widehat{\tfun}(s'|s, \widehat{\pi}(s)) - \tfun(s'|s, \widehat{\pi}(s))]\widehat{Q}(s', \widehat{\pi}(s')) \nonumber \\
& & + \gamma \sum_{s'}\tfun(s'|s, \widehat{\pi}(s))[\widehat{Q}(s', \widehat{\pi}(s')) - Q^*(s', \widehat{\pi}(s'))] \nonumber \\
& \leq & 0 \times \indicator(\event(s, \widehat{\pi}(s))) + 2 \rub \indicator(\event^c(s, \widehat{\pi}(s))) 
+ \gamma \sum_{s'}[\widehat{\tfun}(s'|s, \widehat{\pi}(s)) - \tfun(s'|s, \widehat{\pi}(s))]\widehat{Q}(s', \widehat{\pi}(s')) \nonumber \\
& + & \gamma \sum_{s'}\tfun(s'|s, \widehat{\pi}(s))[\widehat{Q}(s', \widehat{\pi}(s')) - Q^*(s', \widehat{\pi}(s'))], \nonumber
\end{eqnarray}
where the first equality follows from the Bellman equation,  and that $Q^*(s', \pi^*(s')) \geq Q^*(s', \widehat{\pi}(s'))$ for any $s'$, and the second inequality is by the definition of $\event(s, a)$ and Assumption~\ref{con:value-func}. Additionally, by Assumption~\ref{con:value-func}, the second to last line is upper bounded by $2\gamma V_{\max}\textup{TV}_{\widehat{\tfun}}(s, \widehat{\pi}(s))$. This, together with \eqref{eq:mb-term2}, yields that,
\begin{align}\label{eqn:anotherinequality}
        \begin{split}
            \widehat{Q}(s, \widehat{\pi}(s)) - Q^*(s, \widehat{\pi}(s)) 
            \leq&
            2 \rub \indicator(\event^c(s, \widehat{\pi}(s))) +  2\gamma \vub
            \textup{TV}_{\widehat{\tfun}}(s, \widehat{\pi}(s))
            \\
            &+ \gamma \sum_{s'}\prob^{\widehat{\pi}}_{1}(s'|s)\left[\widehat{Q}(s', \widehat{\pi}(s')) - Q^*(s', \widehat{\pi}(s')) \right].
        \end{split}
\end{align}
Applying the same inequality to the last line, we obtain that,
\begin{align*}
         & \widehat{Q}(s, \widehat{\pi}(s)) - Q^*(s, \widehat{\pi}(s))
        \\
        \leq \;
         & 2 \rub \indicator(\event^c(s, \widehat{\pi}(s))) + 2\gamma \vub
        \textup{TV}_{\widehat{\tfun}}(s, \widehat{\pi}(s))
        \\
         & + \gamma \sum_{s'}\prob^{\widehat{\pi}}_{1} (s'|s) \left\{
        2 \rub \indicator(\event^c(s', \widehat{\pi}(s'))) + 2\gamma \vub
        \textup{TV}_{\widehat{\tfun}}(s', \widehat{\pi}(s'))\right\}
        \\
         & + \gamma^2 \sum_{s''}  \prob^{\widehat{\pi}}_{2}(s''|s) \left[\widehat{Q}(s'', \widehat{\pi}(s'')) - Q^*(s'', \widehat{\pi}(s'')) \right],
\end{align*}
Recursively applying this inequality $K$ times yields that,
\begin{align}\label{eqn:recursive}
        \begin{split}
            & \widehat{Q}(s, \widehat{\pi}(s)) - Q^*(s, \widehat{\pi}(s))
            \\
            \leq \;
            & 2\rub \left[\sum_{k=0}^K \gamma^k \sum_{s_k} \prob_k^{\widehat{\pi}}(s_k|s) \indicator(\event^c(s_k, \widehat{\pi}(s_k))) \right]
            \\
            &+ 2\gamma \vub \left[\sum_{k=0}^K \gamma^k \sum_{s_k} \prob_k^{\widehat{\pi}}(s_k|s)\textup{TV}_{\widehat{\tfun}}(s_k, \widehat{\pi}(s_k)) \right]
            \\
            & + \gamma^K \sum_{s_K}  \prob^{\widehat{\pi}}_{K}(s_K|s) \left[\widehat{Q}(s_K, \widehat{\pi}(s_K)) - Q^*(s_K, \widehat{\pi}(s_K)) \right],
        \end{split}
\end{align}
with $s_0=s$ and $\prob_0^{\pi}(s_0|s)=\mathbb{I}(s_0=s)$ for any $s_0$, $s$, $\pi$. Letting $K \rightarrow \infty$, we have that,
\begin{equation}\label{eqn:Kinfity}
        \begin{split}
            \widehat{Q}(s, \widehat{\pi}(s)) - Q^*(s, \widehat{\pi}(s))
            \leq
            \frac{2 \rub}{1-\gamma}  \E_{S \sim d^{\widehat{\pi}}(\cdot|s)} \indicator(\event^c(S, \widehat{\pi}(S))) + \frac{2\gamma \vub}{1-\gamma} \E_{S \sim d^{\widehat{\pi}}(\cdot|s)} \left[
            \textup{TV}_{\widehat{\tfun}}(S, \widehat{\pi}(S))
            \right] .
        \end{split}
\end{equation}
Consequently,
\begin{eqnarray}\label{eqn:onekeyequation}
\E(I_1) & \leq & \frac{2 \rub}{1-\gamma}  \E\left[\E_{S_0 \sim d_0^{\widehat{\pi}}, S \sim d^{\widehat{\pi}}(\cdot|S_0)} \indicator(\event^c(S, \widehat{\pi}(S)))\right]  \\
 & & + \frac{2\gamma \vub}{1-\gamma} \E\left\{ \E_{S_0 \sim d_0^{\widehat{\pi}}, S \sim d^{\widehat{\pi}}(\cdot|S_0)} \left[\textup{TV}_{\widehat{\tfun}}(S, \widehat{\pi}(S)) \right] \right\}.
\end{eqnarray}
Here the outer expectation is taken with respect to three variables: the variability in the estimated policy $\widehat{\pi}$, the probability associated with the event $\event$, and the estimated transition $\widehat{\tfun}$.

Note that the first term on the right-hand-side of \eqref{eqn:onekeyequation} can be upper bounded by
\begin{eqnarray*}
        \frac{2R_{\max}}{1-\gamma}\sup_{\pi}\E\left[\E_{S_0 \sim d_0^{\pi}, S \sim d^{\pi}(\cdot|S_0)} \indicator(\event^c(S, \pi(S)))\right]\le  \frac{2R_{\max}}{1-\gamma}\sup_{\pi,a}\E\left[\E_{S_0 \sim d_0^{\pi}, S \sim d^{\pi}(\cdot|S_0)} \indicator(\event^c(S, a))\right],
\end{eqnarray*}
where the supremum is taken over all deterministic policies. As such, the outer expectation is taken solely with respect to the probability associated with the event $\event$. By Assumption \ref{con:uncertainty}, by interchanging the two expectations, we can upper bound the first term on the right-hand-side of \eqref{eqn:onekeyequation} by
\begin{eqnarray*}
        \frac{2R_{\max}}{1-\gamma}\sup_{\pi,a}\E_{S_0 \sim d_0^{\pi}, S \sim d^{\pi}(\cdot|S_0)} \prob(\event^c(S, a))=\frac{2R_{\max}\alpha}{1-\gamma}.
\end{eqnarray*}
Similarly, we can upper bound the second term on the right-hand-side of \eqref{eqn:onekeyequation} by
\begin{eqnarray*}
        \frac{2\gamma \vub}{1-\gamma} \sup_{\pi} \E\left\{ \E_{S_0 \sim d_0^{\pi}, S \sim d^{\pi}(\cdot|S_0)} \left[
        \textup{TV}_{\widehat{\tfun}}(S, \pi(S))
        \right] \right\}.
\end{eqnarray*}
Since $\rho_0$ has the full support, $\bar{c} \equiv \inf_{s} \rho_0(s)>0$. By Assumption \ref{con:coverage}, the Cauchy-Schwarz inequality and the change-of-measure theorem, it can be further upper bounded by
\begin{eqnarray}\label{eqn:anotherkeyinequality}
        \begin{split}
            &\frac{2\gamma \vub}{1-\gamma} \left(\sup_{\pi}\E\Big\{ \E_{S_0 \sim d_0^{\pi}, S \sim d^{\pi}(\cdot|S_0)} \Big[
            \textup{TV}_{\widehat{\tfun}}(S, \pi(S))
            \Big]^2 \Big\}\right)^{1/2}\\
            \le&\frac{2\gamma \vub}{1-\gamma}\left(\bar{c}^{-1}\sup_{\pi}\E\Big\{ \E_{S_0 \sim \rho_0, S \sim d^{\pi}(\cdot|S_0)} \Big[
            \textup{TV}_{\widehat{\tfun}}(S, \pi(S))
            \Big]^2 \Big\}\right)^{1/2}\\
            =&\frac{2\gamma \vub}{1-\gamma}\left(\bar{c}^{-1}\sup_{\pi}\E\Big\{ \E_{S \sim d^{\pi}} \Big[
            \textup{TV}_{\widehat{\tfun}}(S, \pi(S))
            \Big]^2 \Big\}\right)^{1/2}\\
            \le&\frac{2\gamma \vub}{1-\gamma}\left(\bar{c}^{-1}B_{\mathcal{D}}\E\Big\{ \E_{(S,A)\sim d_{\mathcal{D}}} \Big[
            \textup{TV}_{\widehat{\tfun}}(S, A)
            \Big]^2 \Big\}\right)^{1/2}=\frac{2\gamma V_{\max} \sqrt{B_{\mathcal{D}}}}{\sqrt{\bar{c}}(1-\gamma)}\|\tfun-\widehat{\tfun}\|_{d_{\mathcal{D}}}.
        \end{split}
\end{eqnarray}
Therefore, we obtain that,
\begin{equation}\label{eq:mb-term1}
\E(I_1) \le \frac{2\alpha R_{\max}}{1-\gamma}+\frac{2\gamma V_{\max} \sqrt{B_{\mathcal{D}}}}{\sqrt{\bar{c}}(1-\gamma)}\|\tfun-\widehat{\tfun}\|_{d_{\mathcal{D}}}.
\end{equation}

For $\E(I_2)$, similar to \eqref{eq:mb-term2} and \eqref{eqn:anotherinequality}, we can show that,
\begin{align*}
             & Q^*(s, \pi^*(s)) - \widehat{Q}(s, \pi^*(s))
        \\
        =  \;  & \rfun(s, \pi^*(s)) - \widehat{\rfun}_{\ell}(s, \pi^*(s))
        + \gamma \sum_{s'} \tfun(s'|s, \pi^*(s)) Q^*(s', \pi^*(s'))
        - \gamma \sum_{s'} \widehat{\tfun}(s'|s, \pi^*(s)) \widehat{Q}(s', \widehat{\pi}(s'))
        \\
        \leq \; & \rfun(s, \pi^*(s)) - \widehat{\rfun}_{\ell}(s, \pi^*(s))
        + \gamma \sum_{s'} \tfun(s'|s, \pi^*(s)) Q^*(s', \pi^*(s'))
        - \gamma \sum_{s'} \widehat{\tfun}(s'|s, \pi^*(s)) \widehat{Q}(s', \pi^*(s'))
        \\
        \leq \; & |\widehat{\rfun}_{\ell}(s, \pi^*(s)) - \rfun(s, \pi^*(s))|
        + \gamma \sum_{s'}[\tfun(s'|s, \pi^*(s)) - \widehat{\tfun}(s'|s, \pi^*(s))]\widehat{Q}(s', \pi^*(s'))
        \\ &+ \gamma \sum_{s'}\tfun(s'|s, \pi^*(s))(Q^*(s', \pi^*(s')) -\widehat{Q}(s', \pi^*(s')))
        \\
        \leq \; & |\widehat{\rfun}_{\ell}(s, \pi^*(s)) - \rfun(s, \pi^*(s))|
        + 2\gamma V_{\max} \textup{TV}_{\widehat{\tfun}}(s, \pi^*(s))+ \gamma \sum_{s'}\prob_1^{\pi^*}(s'|s)[Q^*(s', \pi^*(s')) -\widehat{Q}(s', \pi^*(s'))].
\end{align*}
Similar to \eqref{eqn:recursive}, \eqref{eqn:Kinfity} and \eqref{eqn:onekeyequation}, by recursively applying the above inequality, we obtain that
\begin{align*}
        \E I_2
        \leq & \frac{1}{1-\gamma}\E \Big[\E_{S_0\sim d_0^{\widehat{\pi}}, S\sim d^{\pi^*}(\cdot|S_0)} |\widehat{\rfun}_{\ell}(S, \pi^*(S)) - \rfun(S, \pi^*(S))| \Big]
        \\
        &+\frac{2\gamma V_{\max}}{1-\gamma}\E \Big[\E_{S_0\sim d_0^{\widehat{\pi}}, S\sim d^{\pi^*}(\cdot|S_0)} \textup{TV}_{\widehat{\tfun}}(S, \pi^*(S))\Big].
\end{align*}
By similar arguments for \eqref{eqn:anotherkeyinequality}, the right-hand-side can be further upper bounded by
\begin{eqnarray*}
        \frac{\sqrt{B_{\ld}^*}\|\rfun-\widehat{\rfun}_{\ell}\|_{d_{\ld}}}{\sqrt{\bar{c}}(1-\gamma)}+\frac{2\gamma V_{\max} \sqrt{B_{\mathcal{D}}}}{\sqrt{\bar{c}}(1-\gamma)}\|\tfun-\widehat{\tfun}\|_{d_{\mathcal{D}}}.
\end{eqnarray*}
This, together with \eqref{eq:mb-decompo} and \eqref{eq:mb-term1}, yields that
\begin{eqnarray*}
        \E [J(\pi^*)-J(\widehat{\pi})]\le \frac{2\alpha R_{\max}}{(1-\gamma)^2}+\frac{\sqrt{B_{\ld}^*}\|\rfun-\widehat{\rfun}_{\ell}\|_{d_{\ld}}}{\sqrt{\bar{c}}(1-\gamma)^2}+\frac{4\gamma V_{\max} \sqrt{B_{\mathcal{D}}}}{\sqrt{\bar{c}}(1-\gamma)^2}\|\tfun-\widehat{\tfun}\|_{d_{\mathcal{D}}},
\end{eqnarray*}
leading to the desired regret bound. This completes the proof of Theorem~\ref{thm:mb}.
\end{proof}

\subsection{Proof of Corollary~\ref{coro1}}
\begin{proof}
For Algorithm~\ref{alg:mb-sorl}, Theorem~\ref{thm:mb} establishes that,
\begin{eqnarray*}
        \E [J(\pi^*)-J(\widehat{\pi})]\le \underbrace{\frac{4\gamma V_{\max} \sqrt{B_{\mathcal{D}}}}{\sqrt{\bar{c}}(1-\gamma)^2}\|\tfun-\widehat{\tfun}\|_{d_{\mathcal{D}}}}_{I_{11}}+\underbrace{\frac{2\alpha R_{\max}}{(1-\gamma)^2}}_{I_{12}}+\underbrace{\frac{\sqrt{B_{\ld}^*}\|\rfun-\widehat{\rfun}_{\ell}\|_{d_{\ld}}}{\sqrt{\bar{c}}(1-\gamma)^2}}_{I_{13}}.
\end{eqnarray*}
With probability at least $n^{-1}$, we have that,
    \begin{align*}
        I_{11} \leq \sqrt{\frac{2|\sspace|^3|\aspace|\log(n |\sspace||\aspace|)}{n}}.
    \end{align*}
Since $|\ud|$ is infinite, then $I_{11} = o_p(1)$. Because $\alpha$ is sufficient small, we have $I_{12} \leq I_{13}$. Recognizing that $\gamma$ and $\bar{c}$ are some constants, we have that,
    \begin{align*}
        \E [J(\pi^*)-J(\widehat{\pi})] = O_p\left(\sqrt{B_{\ld}^*}\|\rfun-\widehat{\rfun}_{\ell}\|_{d_{\ld}}\right).
    \end{align*}

Next, for Algorithm~\ref{alg:mf-sorl}, Theorem~\ref{thm:mf-fqi} establishes that,
    \begin{align*}
        \E[J(\pi^*) - J(\widehat{\pi})]
        \leq&
        \underbrace{\frac{2\gamma^K V_{\max}}{1-\gamma}}_{I_{21}}+ \underbrace{\frac{2\alpha R_{\max}}{(1-\gamma)^2}}_{I_{22}}+\underbrace{\frac{c_1\sqrt{B_{\ld}^*}}{(1-\gamma)^2}\|\rfun-\widehat{\rfun}_{\ell}\|_{d_{\ld}}}_{I_{23}}
        \\
        &+\underbrace{\frac{c_1V_{\max}\sqrt{B_{\mathcal{D}}}}{(1-\gamma)^2\sqrt{n}}\sqrt{\ln (|\mathcal{F}||\mathcal{Q}|n)}}_{I_{24}}.
    \end{align*}
When $K$ is sufficiently large and $\alpha$ is sufficiently small, we have $I_{21} \leq I_{23}$ and $I_{22} \leq I_{23}$. Additionally, for $I_{24}$, it equals 0 because $|\ud|$ is infinite. Together we have that,
    \begin{align*}
        \E[J(\pi^*) - J(\widehat{\pi})] = O\left( \sqrt{B_{\ld}^*}\|\rfun-\widehat{\rfun}_{\ell}\|_{d_{\ld}}\right).
    \end{align*}
    This completes the proof of Corollary~\ref{coro1}.
\end{proof}

\section{Numerical experiments}
\label{sec:experiments-details}

\subsection{Details on the illustrative example in Section \ref{sec:key-ideas}}
\label{sec:detail-toy-env}

\textbf{Environment.} We consider the following synthetic environment.
\vspace{-0.5em}
\begin{itemize}
    \item The state space: $\sspace = \{0, 1, 2\} \times \{0, 1, 2\}$. The neighbour set of $s \in \sspace$ is defined as $\mathcal{N}(s) = \{b \in \sspace \mid \|s - b\|_1 \leq 1\}$.

    \item The action space: $\aspace = \{(0, 0), (0, 1), (1, 0), (0, -1), (-1, 0)\}$.

    \item The reward $R$ is generated as follows: 
          \begin{eqnarray*}
              R\sim 
              \begin{cases}
                  \mathcal{N}(10, 1),    & \textup{ if state } s \textup{ transits to }(0, 0) \textup{ after executing action } a;
                  \\
                  \mathcal{N}(-0.1, 10), & \textup{ otherwise}.
              \end{cases}
          \end{eqnarray*}

    \item The transition function $\tfun$:
          \vspace{-0.5em}
          \begin{eqnarray*}
              \tfun(s'|s, a)=
              \begin{cases}
                  0.9,                  & \textup{if } s'=s+a                \\
                  0.1/|\mathcal{N}(s)|, & \textup{if } s' \in \mathcal{N}(s) \\
                  0.0,                  & \textup{otherwise}
              \end{cases}.
          \end{eqnarray*}
\end{itemize}
\vspace{-0.5em}

\noindent \textbf{Learning $\rfun_{\ell}$.} We estimate the reward function $\rfun(s, a)$ by the conditional empirical mean,
\begin{align*}
    \widehat{\rfun}(s, a) = \frac{1}{|\rfun_{s, a}|} \sum_{(\tilde{s}, \tilde{a}, \tilde{r}, \tilde{s}') \in \rfun_{s, a}} \tilde{r},
\end{align*}
where $\rfun_{s, a} = \{(\tilde{s}, \tilde{a}, \tilde{r}, \tilde{s}') \in \ld | \tilde{s}=s, \tilde{a} = a\}$. We estimate the standard deviation of reward at $(s, a)$ as,
\begin{align*}
    \Delta(s, a) = \sqrt{\frac{1}{|\rfun_{s, a}|} \sum_{(\tilde{s}, \tilde{a}, \tilde{r}, \tilde{s}') \in \rfun_{s, a}} (\tilde{r} - \widehat{\rfun}(s, a))^2}.
\end{align*}
We then set  $\widehat{\rfun}_{\ell}(s, a) = \hat{r}(s, a) - z_{1-\alpha/2}\times \Delta(s, a)$, where $z_{1-\alpha/2}$ is the $(1-\alpha/2)$-th quantile of standard normal random variable. We set $\alpha = 0.05$, and $z_{1-\alpha/2} \approx 2.0$.

\noindent \textbf{Learning $Q^*$.} We employ the classical tabular Q-learning algorithm \citep{sutton2018reinforcement} to learn the Q table of the optimal policy $\pi^*$. We repeat the update rule for Q table 60000 times with the learning rate $0.001$, and set the discount factor as $\gamma=0.95$.

\noindent \textbf{Evaluating the learned policy.} For this example, we assess the greedy policy derived from the estimated optimal Q-function $\widehat{Q}$, using the cumulative reward $\E^{\widehat{\pi}}[\rfun(S_1, A_1) + \rfun(S_2, A_2)]$ as the evaluation criterion. We estimate this criterion by rolling out 500 trajectories, each with two time points, under the policy $\widehat{\pi}$, then computing the cumulative reward as ${500}^{-1}\sum_{i=1}^{500}[R_{i1}+R_{i2}]$.

\subsection{Details on the synthetic environment in Section \ref{sec:synthetic-env}}
\label{sec:detail-synthetic-env}

\textbf{Environment.} We consider the following synthetic environment.
\vspace{-0.5em}
\begin{itemize}
     \item The state space: $\sspace = \mathbb{R}^2$.

     \item The action space: $\aspace = \{-1, 0, 1\}$.

     \item The reward function: for any $s=(s_1, s_2) \in \sspace$ and $a \in \aspace$, the reward $\rfun(s, a)$ is a Gaussian random variable $\mathcal{N}(\mu(s, a), \sigma^2(s, a))$,  with mean $\mu(s, a) = 5 a \times (s_1+s_2)$ and standard deviation $\sigma(s, a) = 0.8$ if $a=0$, and $\sigma(s, a) = 0.1$ if $a\neq 0$. As such, when $(s_1 + s_2) / 2 > 0$, $a = 1$ is the optimal action; otherwise, $a = -1$ is the optimal action.

     \item The transition function $\tfun(\cdot|s, a)$: it follows a conditional Gaussian distribution with mean $\mu'(s, a)$:
           \begin{eqnarray*}
              \mu'(s, a) =
               a\begin{pmatrix}
                   -0.77, & 0.23 \\
                   0.23,  & 0.77
               \end{pmatrix} s
           \end{eqnarray*}
           and covariance $0.01 \times \mathbf{I}_{2\times 2}$. Here, $\mathbf{I}_{2\times 2}$ is a 2-by-2 identical matrix.
\end{itemize}

\noindent \textbf{Generating $\ld$ and $\ud$.} We generate $\ld$ using a random policy to roll out multiple trajectories, each with a horizon of 30 time points. Similarly, we generate $\ud$ with $10 \times n$ trajectories, but without recording the rewards. This way, $\ld$ has a full coverage on $\sspace \times \aspace$. To simulate the case when $\ld$ has a limited action coverage, we remove 80\% of the tuples in $\ld$ with sub-optimal actions.

\noindent \textbf{Learning $\rfun_{\ell}$.}
We first encode the discrete actions as a two-dimensional one-hot vector $(a_1, a_2) \in \{0, 1\}^2$. Then, given $s=(s_1, s_2) \in \sspace$ and one-hot-encoding action $(a_1, a_2)$, we compute all polynomial combinations of $(s_1, s_2, a_1, a_2)$ with degree up to 2, resulting in a $9$-dimensional feature vector $g: \sspace \times \aspace \to \mathbb{R}^9$. Given $\ld$, we fit OLS for the reward $R$ given $g(s, a)$, and obtain the estimated reward. Next, we compute $\Delta_{\textup{\sug}}(s, a)$. In practice, for each $(s,a)$ pair in the unlabeled data, we compute $\Delta_{\sug}(s,a)$ and retain only those pairs whose value falls below the $q$th quantile, where $q$ is set to 0.9 under full coverage and 0.3 under partial coverage.

\noindent\textbf{Learning $\pi^*$.}
We use FQI to learn $\pi^*$ with the discount rate $\gamma = 0.99$, and the maximum iteration equal to 500. We stop the iterations either when the number of iterations reaches $K$, or $\sum_{(S, A) \in \ld} |Q^{(k+1)}(S, A) - Q^{(k)}(S, A)| \leq 10^{-6} \times \sum_{(S, A) \in \ld} |Q^{(t)}(S, A)|$.

\noindent\textbf{Evaluating the learned policy.} We evaluate the learned policy $\widehat{\pi}$ by computing the regret, i.e., $\E\left[J(\widehat{\pi})\right] - \E\left[J(\pi^*)\right]$. To achieve this, we approximate the $\E\left[J(\pi^*)\right]$ and $\E\left[J(\widehat{\pi})\right]$ by Monte Carlo estimation. Specifically, given a policy $\pi$, we generate $N$ trajectories with $T$ time points. and approximate $\E[J(\pi)]$ by $\frac{1}{N}\sum_{i=1}^N\sum_{t=1}^T \gamma^t R_t$. In our experiments, we set $N=100$ and $T=20$.

\subsection{Details on the MuJoCo environments}
\label{sec:detail-mujoco-env}

\textbf{Environment.} We consider the version `-v2'' of the MuJoCo environments, and we consider three datasets, halfcheetah, walker2d, and hopper. In addition, we consider three type of unlabeled datasets: 
\vspace{-0.5em}
\begin{itemize}
\item ``full-replay'': contains the replay buffer from a fully trained policy; 
\item ``medium-replay'': contains the replay buffer when a policy is partially-trained; 
\item ``medium'': contains the tuples samples from a partially-trained policy. 
\end{itemize}

\noindent\textbf{Generating $\ld$ and $\ud$.} We generate $\ld$ and $\ud$ based on the D4RL benchmark data \citep{fu2020d4rl}. Specifically, we randomly sample 10k transitions from the ``expert'' dataset to serve as $\ld$, and we create $\ud$ by using transitions from ``full-replay'', ``medium-replay'', or ``medium'' datasets.

\noindent\textbf{Learning $\rfun_{\ell}$.} To account for the nonlinear dependence between $(s, a) \in \sspace \times \aspace$ and $r \in \mathbb{R}$, we set $g(\cdot, \cdot)$ to an approximated radial-basis-function kernel feature map using random Fourier features \citep{rahimi2008weighted}. Similar to Appendix~\ref{sec:detail-synthetic-env},  we apply OLS to the reward $R$ given $g(s, a)$, and we use random forest for the auxiliary reward $\widehat{R}_{\aux}$. 

\noindent\textbf{Learning $\tfun$.} We model $\tfun$ using an ensemble of $K$ single-hidden-layer neural networks $\{ \widehat\tfun^{(k)}(\cdot|s, a) = \mathcal{N}(\mu^{(k)}(s, a), \Sigma^{(k)}(s, a)) \}_{k=1}^K$, each outputting the mean and variance of a multivariate Gaussian distribution over the next state. Due to the large dimensionality of $\Sigma^{(1)}, \ldots, \Sigma^{(K)}$, in our implementation, we set the covariance matrices $\{\Sigma^{(k)}\}_{k=1}^K$ as the diagonal matrices. Each neural network's hidden layer contains 100 units. One neural network from the ensemble is randomly chosen to roll out the observations of the next state. A critical aspect of rolling out the samples is the maximum time spent interacting with the environment, i.e., the maximum length of trajectories generated by $\widehat{\rfun}_{\text{SPL}}$ and $\widehat{\tfun}$. Following the recommendation in \citet{yu2020mopo}, we interact with the estimated environment $k$ times. Similar to \citet{yu2020mopo}, we have found that setting $k=5$ for the halfcheetah and hopper environments, and setting $k=1$ for the walker2d environment, leads to good empirical performance.

\noindent\textbf{Learning $\pi^*$.} We adopt the network architecture in \citet{haarnoja2018sac}. Specifically, the critic network is a three-layer MLP with 256 hidden nodes and ReLU activation. The policy network is a Gaussian policy network that enables automatic entropy tuning. Parameters in these neural networks are trained using the Adam optimizer \citep{kingma2014adam}, with the learning rate set to $3\times10^{-4}$. At each update, we randomly draw 256 samples from the replay buffer for parameter estimation. Besides, we use the target network trick to train the critic: the parameters in the target critic network is updated by $w' \leftarrow (1 - \tau)w' + \tau w$ where $w$ is the corresponding parameter in the critic network, and $\tau$ is fixed at $5 \times 10^{-3}$. Finally, we set the discount rate $\gamma = 0.99$. 

\noindent\textbf{Evaluating the learned policy.} The criterion for assessing a policy $\pi$ is the normalized cumulative reward averaged over five random seeds. Computation involves two steps. First, we compute the non-normalized cumulative reward, i.e., $\E^\pi\left[\sum_{t=1}^{T} \rfun(S_t, A_t)\right]$, where $T$ is fixed at 1000. This is approximated using Monte Carlo estimation by generating 10 trajectories, each with $T = 1000$ time points. Next, these cumulative rewards $\widehat{S}$ are normalized by $(\widehat{S} - S_{\min})/(S_{\max} - S_{\min})$, where $S_{\min}$ represents the cumulative reward of a random policy, and $S_{\max}$ corresponds to that of a fully trained policy.

\end{document}